\documentclass{article}

\PassOptionsToPackage{numbers, compress, sort}{natbib}

\usepackage[preprint]{neurips_2023}

\usepackage[utf8]{inputenc} 
\usepackage[T1]{fontenc}    
\usepackage{hyperref}       
\usepackage{url}            
\usepackage{booktabs}       
\usepackage{amsfonts}       
\usepackage{nicefrac}       
\usepackage{microtype}      
\usepackage{xcolor}         

\usepackage{mathrsfs}
\usepackage{bbold}
\usepackage{stmaryrd}
\usepackage{tikz}
\usetikzlibrary{positioning}
\usepackage{subcaption}
\usepackage{algorithmic}
\usepackage{algorithm}

\usepackage{amsmath}
\usepackage{amssymb}
\usepackage{mathtools}
\usepackage{amsthm}
\usepackage{dsfont}

\usepackage{enumitem}

\usepackage{color}
\definecolor{matthieu}{RGB}{0, 127, 255}

\theoremstyle{plain}
\newtheorem{theorem}{Theorem}[section]
\newtheorem{proposition}[theorem]{Proposition}
\newtheorem{lemma}[theorem]{Lemma}

\theoremstyle{definition}
\newtheorem{definition}[theorem]{Definition}

\newtheorem{example}[theorem]{Example}
\newtheorem{remark}[theorem]{Remark}

\newtheorem{innercustomthm}{Theorem}
\newenvironment{customthm}[1]
  {\renewcommand\theinnercustomthm{#1}\innercustomthm}
  {\endinnercustomthm}

\newtheorem{innercustomprop}{Proposition}
\newenvironment{customprop}[1]
  {\renewcommand\theinnercustomprop{#1}\innercustomprop}
  {\endinnercustomthm}

\DeclareMathOperator*{\argmin}{arg\,min}

\newcommand{\fngw}{\mathrm{FNGW}_{\alpha, \beta}}

\title{Exploiting Edge Features in Graphs with Fused Network Gromov-Wasserstein Distance}

\author{%
  Junjie Yang\,$^1$ \qquad Matthieu Labeau\,$^1$ \qquad Florence d'Alché-Buc\,$^1$ \\[0.12cm]
  $^1$\,LTCI, Télécom Paris, Institut Polytechnique de Paris, France \\
  Correspondence to: \texttt{\href{mailto:junjie.yang@telecom-paris.fr}{junjie.yang@telecom-paris.fr}}
}

\begin{document}

\maketitle

\begin{abstract}
Pairwise comparison of graphs is key to many applications in Machine learning ranging from clustering, kernel-based classification/regression and more recently supervised graph prediction.
Distances between graphs usually rely on informative representations of these structured objects such as bag of substructures or other graph embeddings.
A recently popular solution consists in representing graphs as metric measure spaces, allowing to successfully leverage Optimal Transport, which provides meaningful distances allowing to compare them: the Gromov-Wasserstein distances. However, this family of distances overlooks edge attributes, which are essential for many structured objects.
In this work, we introduce an extension of Gromov-Wasserstein distance for comparing graphs whose both nodes and edges have features. We propose novel algorithms for distance and barycenter computation. We empirically show the effectiveness of the novel distance in learning tasks where graphs occur in either input space or output space, such as classification and graph prediction.
\end{abstract}

\section{Introduction}
Optimal Transport (OT)~\citep{villani_optimal_2009} has recently offered novel insights on graph modeling, benefiting a series of tasks, e.g. graph classification~\citep{vayer_optimal_2019}, graph matching~\citep{xu_gromov-wasserstein_2019}, dictionary learning~\citep{vincent-cuaz_online_2021} or structured prediction~\citep{brogat-motte_learning_2022}. 
Based on the representation of a graph as a metric measure space, where the nodes of the graph are considered as the support of the probability measure, OT provides a natural way to compute a meaningful distance between graphs, such as the Gromov-Wasserstein (GW) distance~\citep{memoli_gromovwasserstein_2011, sturm_space_2012}. The asymmetric structure of directed graphs has led to the generalization of GW via the Network Gromov-Wasserstein (NGW) distance~\citep{chowdhury_gromovwasserstein_2019} while the Fused Gromov-Wasserstein (FGW) distance~\citep{vayer_optimal_2019} was proposed to extend GW to node-labeled graphs. Recently, \citet{barbe_graph_2021} proposed the DifFused Gromov Wasserstein (DFGW) distance to smooth the node features along the graph structure through a heat diffusion operator which uses the Laplacian kernel. 
Overall, this new class of distances enjoys various 
useful metric and geodesic properties, allowing for example to derive the minimum path between two structured objects~\citep{sturm_space_2012, vayer_fused_2020}.

On the other hand, many machine learning applications require dealing with graphs with complex edge features, such as the abstract meaning representation (AMR)~\citep{banarescu_abstract_2013} in natural language processing or the scene graph~\citep{j_johnson_image_2015} in computer vision, where they model a great deal of possible relations between entities, or the chemical bonding types between atoms of a molecule in chemistry~\citep{irwin_zinc_2012, duhrkop2015searching}. 
Dynamic graphs may also require more sophisticated modeling, as they use temporal edge features which continuously evolve over time~\citep{kazemi_representation_2020}.
There have been several attempts in the literature to obtain meaningful graph representations by including edge features: the two dominant solutions are to use graph kernels 
and Graph Neural Networks (GNNs). Among graph kernels, 
the Neighborhood Subgraph Pairwise Distance Kernel (NSPDK)~\citep{costa_fast_2010} takes edge labels into consideration to build graph-invariant encodings of subgraphs while more involved graph kernels leverage bags of subgraphs~\citep{grenier2015bags}.
The GNNs belonging to the Message Passing Neural Networks (MPNNs) framework~\citep{simonovsky_dynamic_2017, gilmer_neural_2017, fey_splinecnn_2018, corso_principal_2020} incorporate directly edge features via their aggregation procedure, while attention-based GNNs~\citep{velickovic_graph_2018, shi_masked_2021, brody_how_2022} may leverage edge features to compute the attention weights. 
Building on recent computational and theoretical results on GW distances, this paper proposes to represent the edge features of a graph by equipping the original measure space with an additional binary function whose codomain falls in a metric space. 
This generalization allows us to flexibly include edge features into the computation of a GW distance 
while keeping its desirable topological properties.
It especially unlocks better graph modeling for a wide array of tasks, of which we explore several. We propose dedicated procedures for two types of applications: (i) distance-based learning where, for example, the distance is used to define a kernel in the input space to solve graph classification tasks and (ii) barycenter-based learning for which the solution of the problem writes as a \textit{graph barycenter}. For the latter, we focus on two emblematic tasks, dictionary learning and supervised graph prediction.

We verify the interest of the proposed distance on synthetic and real-world datasets. In particular, we provide extensive results on real-world graph classification benchmark tasks as well as a supervised graph prediction problem under the form of a metabolite identification task. 

\paragraph{Summary of contributions:}
\begin{itemize}[leftmargin=*]
    \item We propose a new OT-based distance (FNGW), which is a generalization of both the Network Gromov-Wasserstein Distance and the Fused Gromov-Wasserstein Distance to edge-labeled graphs. 
    \item We derive an algorithm for the computation of the FNGW distance in the discrete case, along with procedures dedicated to  barycenter computation, dictionary learning and graph prediction based on the FNGW distance.     \item We carry out experiments on synthetic datasets where graphs are directed and both nodes and edges are labeled, followed by classification and structured prediction experiments on real-world graph datasets, for which we observe a significant increase in performance compared to the FGW distance and state-of-the-art kernel-based methods.
\end{itemize}

\paragraph{Notations:}
For two probability measures $\mu \in \mathrm{Prob}(X_0)$, $\nu \in \mathrm{Prob}(X_1)$ where $X_0$ and $X_1$ are both polish spaces, we note $\Pi(\mu, \nu)$ the set of all couplings of $\mu$ and $\nu$, i.e., the set of probability measures $\pi$ on the product space $X_0 \times X_1$ satisfying $\pi(A_0, X_1) = \mu(A_0)$ and $\pi(X_0, A_1) = \nu(A_1)$ for all $A_0 \in \mathcal{B}(X_0)$, $A_1 \in \mathcal{B}(X_1)$ where $\mathcal{B}(\cdot)$ denotes the Borel $\sigma$-algebra. $\Sigma_n = \{\boldsymbol{h} \in (\mathbb{R}_+)^n, \sum_{i=1}^n h_i = 1\}$ is the simplex histogram with $n$ bins.
In the case where both $\mu$ and $\nu$ are discrete, 
i.e. we can write $\mu = \sum_{i=1}^n a_i\delta_{x_{i}^0}$ and  $\nu = \sum_{j=1}^m b_j\delta_{x_j^1}$ with $\boldsymbol{a} \in \Sigma_n$ and $\boldsymbol{b} \in \Sigma_m$. Here $\delta$ denotes the dirac measure. We note $\Pi(\boldsymbol{a}, \boldsymbol{b})$ the set of matrices $\pi \in \mathbb{R}^{n\times m}$ satisfying $\sum_i \pi_{i,j} = b_j$ and $\sum_j \pi_{i,j} = a_i$ despite the abuse of the notation. We use $\#$ to denote the pushforward operator on measures.
We note $\times_n$ the $n$-mode tensor-matrix product. Given a tensor $X\in \mathbb{R}^{I_1\times I_2 \cdots \times I_N}$ and a matrix $A \in \mathbb{R}^{J\times I_n}$, $X \times_n A$ gives a tensor of shape $I_1\times \cdots \times I_{n-1}\times J \times I_{n+1}\times \cdots \times I_N$ where $(X \times_n A)(i_1,\cdots,i_{n-1},j,i_{n+1},\cdots,i_N) = \sum_{i_n=1}^{I_n} X(i_1, \cdots, i_n,\cdots, i_N)A(j,i_n)$. 
We note $\mathcal{I}_{N\times M}$ a tensor of shape $N\times N \times M$ with $\mathcal{I}_{N\times M}(n_1, n_2, m) = \hat{\delta}_{n_1n_2}$ where $\hat{\delta}$ is the Kronecker delta function.
\section{Fused Network Gromov-Wasserstein Distance}
In this section, we first give the general definition of the Fused Network Gromov-Wasserstein Distance that extends both the NGW-distance and the FGW-distance whose definitions are recalled in Appendix A. We then derive its discrete form, adapted to the representation of graphs, with the corresponding computation algorithm. Proofs of propositions and theorems are given in Appendix A.
\begin{definition}[Fused Network Gromov-Wasserstein Distance]
	Let $\mathcal{G}$ be the set of tuples of the form $(X, \psi_{X}, \varphi_X, \omega_{X}, \mu_{X})$ where $X$ is a polish space, $\psi_{X}: X\to \Psi $ is a bounded continuous measurable function from $X$ to a metric space $(\Psi, d_\Psi)$ , $\varphi_X: X \times X \to \mathbb{R}$ is a bounded continuous measurable function, $\omega_{X}: X\times X \to \Omega$ is a bounded continuous measurable function from $X^2$ to a metric space $(\Omega, d_\Omega)$ and $\mu_{X}$ is a fully supported Borel probability measure.   
	Given two tuples $g_X = (X, \psi_{X}, \varphi_X, \omega_{X}, \mu_{X})$, $g_Y=(Y, \psi_{Y}, \varphi_Y, \omega_{Y}, \mu_{Y})$ from $\mathcal{G}$ and trade-off parameters $(\alpha, \beta) \in [0, 1]^2$, the Fused Network Gromov-Wasserstein Distance between $g_X$ and $g_Y$ is defined for any $(p, q)\in [1, \infty]$ as follows:
	\begin{gather}
		\mathrm{FNGW}_{\alpha, \beta, q, p}(g_X, g_Y) = \min_{\mu\in \Pi(\mu_X, \mu_Y)}\mathcal{E}_{\alpha, \beta, q, p}(g_X, g_Y, \mu)
	\end{gather}
    with 
    \begin{align}
        \mathcal{E}_{\alpha, \beta, q, p}(g_X, g_Y, \mu) = \bigg( \int_{X\times Y}\int_{X\times Y} & [(1-\alpha-\beta) d_{\Psi}\left(\psi_X(x),\psi_Y(y)\right)^q  \nonumber \\ + \alpha d_\Omega(\omega_{X}(x, x'),\omega_{Y}(y, y'))^q & + \beta |\varphi_{X}(x, x') - \varphi_{Y}(y, y')|^q] ^p 
        d\mu(x,y)d\mu(x',y') \bigg)^{\frac{1}{p}}
    \end{align}
 \label{def:fngw}
\end{definition}

\begin{remark}
    Note that $\varphi$ is not necessarily a symmetric function, which corresponds to the setting of the NGW~\citep{chowdhury_gromovwasserstein_2019}, as opposed to the GW distance. We do not impose that constraint on the newly introduced function $\omega$ either.
\end{remark}

\begin{example}[Directed Labeled Graph]
\label{ex:dlg}
A graph can be represented by a tuple $(X, \psi_{X}, \varphi_X, \omega_{X}, \mu_{X})$ where $X$ is the set of all the nodes of the graph, $\psi$ maps each node into its node feature space $\Psi$, $\varphi(x_i, x_j)$ measures the presence of an \textbf{explicit} directed edge from node $x_i$ to $x_j$, $\omega$ maps each \textbf{implicit} directed edge (of which we always assume the existence) associated to two nodes into an edge feature space $\Omega$. Finally, $\mu_X$ attributes importance weights to the nodes of the graph.
\end{example}
With this representation of graphs, the distance introduced in Def. \ref{def:fngw} can naturally serve as a \textit{metric}. 
The following theorems state the existence of the FNGW distance and its metric properties.
\begin{theorem}[Optimal Coupling]
	Given  $g_X = (X, \psi_{X}, \varphi_X, \omega_{X}, \mu_{X})$,  $g_Y=(Y, \psi_{Y}, \varphi_Y, \omega_{Y}, \mu_{Y})$, for any $(p, q) \in  [1, \infty]$ and  $(\alpha, \beta) \in [0, 1]^2$, there exists an optimal coupling $\mu^*\in \Pi(\mu_X, \mu_Y)$ which satisfies $\mathrm{FNGW}_{\alpha, \beta, q, p}(g_X, g_Y) = \mathcal{E}_{\alpha, \beta, q, p}(g_X, g_Y, \mu^*)$. 
 \label{theorem:opt_couple}
\end{theorem}

\begin{theorem}[Metric Properties]
	The FNGW distance satisfies the following properties:
	for all $g_X = (X, \psi_{X}, \varphi_X, \omega_{X}, \mu_{X})$, $g_Y=(Y, \psi_{Y}, \varphi_Y, \omega_{Y}, \mu_{Y})$ and $g_Z=(Z, \psi_{Z}, \varphi_Z, \omega_{Z}, \mu_{Z})$ from $  \mathcal{G}$:
	\begin{itemize}[leftmargin=*]
		\item  $\mathrm{FNGW}_{\alpha, \beta, q, p}(g_X, g_Y) \geq 0 $
        \item $\mathrm{FNGW}_{\alpha, \beta, q, p}(g_X, g_Y) = \mathrm{FNGW}_{\alpha, \beta, q, p}(g_Y, g_X)$
        \item $\mathrm{FNGW}_{\alpha, \beta, q, p}(g_X, g_X) = 0 $. Moreover, $\mathrm{FNGW}_{\alpha, \beta, q, p}(g_X, g_Y) = 0 $ if and only there is a Borel probability space $(Z, \mu_Z)$ with measurable maps $f: Z\to X$ and  $g: Z\to Y$ such that
			\begin{gather}
		\nonumber f_{\#}\mu_Z = \mu_X \quad g_{\#}\mu_Z = \mu_Y \\
		\nonumber \| (1-\alpha - \beta) d_{\Psi}\left(\psi_X\circ f,\psi_Y\circ g\right)^q  + \alpha d_\Omega(f^{\#}\omega_{X}, g^{\#}\omega_{Y})^q + \beta |f^{\#}\varphi_X - g^{\#}\varphi_Y|^q\|_{\infty} = 0
	\end{gather}
    where $f^{\#}\omega_{X}: Z \times Z \to \Omega$ is the pullback weight function defined by $(z, z') \to \omega_{X}(f(z), f(z'))$ and  $f^{\#}\varphi_{X}: Z \times Z \to \mathbb{R}$ is given by $(z, z') \to \varphi_{X}(f(z), f(z'))$. 
        \item $\mathrm{FNGW}_{\alpha, \beta, q, p}(g_X, g_Z) \leq  2^{q-1}(\mathrm{FNGW}_{\alpha, \beta, q, p}(g_X, g_Y) + \mathrm{FNGW}_{\alpha, \beta, q, p}(g_Y, g_Z))$
	\end{itemize}
 \label{theorem:metric}
\end{theorem}

We now give the definition of the FNGW when the measure $\mu_X$ is discrete, following Example~\ref{ex:dlg}.
\begin{definition}[FNGW Distance in the Discrete Case]
Let $(X, \psi_{X}, \varphi_{X}, \omega_{X}, \mu_{X}) \in \mathcal{G}$ and suppose $X$ is a finite set of size $n$. Let $F\in \Psi^n$ be the 
set of values of $\psi_X$ on every point of $X$, $A \in \mathbb{R}^{n\times n}$ and $E \in \Omega^{n\times n}$ be the respective sets for $\varphi_X$ and $\omega_X$ on every pair of points of $X^2$. Lastly, let $\boldsymbol{p} \in \Sigma_n$ be the histogram which satisfies $\mu_X = \sum_i^n p_i\delta_{x_i}$. Given two quadruples $g = (F, A, E, \boldsymbol{p})$ of size $n$, $\tilde{g} = (\tilde{F}, \tilde{A}, \tilde{E}, \tilde{\boldsymbol{p}})$ of size $m$ corresponding to two tuples of $\mathcal{G}$, and trade-off parameters $(\alpha, \beta) \in [0, 1]^2$, for  $(p, q) \in  [1, \infty]$ the Fused Network Gromov-Wasserstein distance between them is written as:
	\begin{equation}
		\mathrm{FNGW}_{\alpha, \beta, q, p}(g, \tilde{g}) = \min_{\pi \in \Pi(\boldsymbol{p},\tilde{\boldsymbol{p}})} \mathcal{E}_{\alpha, \beta, q, p}(\{F, A, E\}, \{\tilde{F}, \tilde{A}, \tilde{E}\}, \pi)
	\label{eq/fngw}
	\end{equation}
	with
	\begin{align}
		 \nonumber \mathcal{E}_{\alpha, \beta, q, p}(\{F, A, E\}, \{\tilde{F}, \tilde{A}, \tilde{E}\}, \pi) = \bigg( \sum_{i,j,k,l}\bigg[ & \alpha d_{\Omega}\left(E(i,k), \tilde{E}(j,l)\right)^q + \beta |A(i, k)- \tilde{A}(j,l)|^q \\  
        &  + (1-\alpha -\beta) d_{\Psi}\left(F(i),\tilde{F}(j)\right)^q \bigg]^p\pi_{k,l}  \pi_{i,j}\bigg)^{\frac{1}{p}}
	\end{align}
	We define the 4-dimensional tensors $J(A, \tilde{A})$ and $L(E, \tilde{E})$ as follows:
	\begin{gather}
        J_{i, j, k, l}(A, \tilde{A}) = |A(i,k) - \tilde{A}(j,l)|^q \quad L_{i, j, k, l}(E, \tilde{E}) = d_{\Omega}\left(E(i,k), \tilde{E}(j,l)\right)^q
	\end{gather}
	and the cost matrix $M(F, \tilde{F})$:
	\begin{equation}
		M_{i,j}(F, \tilde{F}) =  d_{\Psi}\left(F(i),\tilde{F}(j)\right)^q
	\end{equation}
	Choosing $p=1$, we can rewrite
     \begin{gather}
        \mathcal{E}_{\alpha, \beta, q, p} = 
        \langle(1-\alpha-\beta)M(F, \tilde{F}) + \beta J(A, \tilde{A})   \otimes\pi + \alpha L(E, \tilde{E}) \otimes\pi, \pi\rangle
     \end{gather}
\end{definition}

\begin{remark}[Computational Complexity]
    Following~\citet{vayer_optimal_2019}, the term $J(A, \tilde{A})$ can be computed efficiently with an appropriate choice of metric space $\Psi=\mathbb{R}^S$ with $d_{\Psi}(a, b) = \|a-b\|_{\mathbb{R}^S}$, and $q = 2$.
	The computation of $L(E, \tilde{E}) \otimes\pi$ is similarly non-trivial, requiring $O(n^2m^2T)$ operations, where $T$ is the number of operations necessary to compute the distance $d_{\Omega}(E(i,k), \tilde{E}(j,l))$. However, by choosing again an appropriate metric space $(\Omega, d_{\Omega})$ and $q$, its complexity can be reduced:
\label{remark:complexity}
\end{remark}

\begin{proposition}
	When $\Omega = \mathbb{R}^T$  with its associated metric $d_{\Omega}(a, b) = \|a-b\|_{\mathbb{R}^T}$ and $q=2$, the term $ L(E, \tilde{E}) \otimes\pi$ becomes
	\begin{equation}
		L(E, \tilde{E}) \otimes\pi = g(E)\boldsymbol{p}\mathbb{1}_m^\mathsf{T} + \mathbb{1}_n\boldsymbol{\tilde{p}}^\mathsf{T}h(\tilde{E})^\mathsf{T} -2 \sum_{t=1}^{T}E[t]\pi\tilde{E}[t]^\mathsf{T} 
    \label{eq:tensor_matrix}
	\end{equation}
	 where $g: \mathbb{R}^{n\times n\times T} \to \mathbb{R}^{n\times n}$ is expressed as $g(E)_{i,j}= \|E(i,j)\|^2_{\mathbb{R}^T}$, $h: \mathbb{R}^{m\times m\times T} \to \mathbb{R}^{m\times m}$ is expressed as by $h(\tilde{E})_{i,j}= \|\tilde{E}(i,j)\|^2_{\mathbb{R}^T}$ and the matrix $E[t](i, j) = E(i, j, t)$ for any $i, j, t$. It can hence be computed with the complexity $O(n^2mT+nm^2T)$.
\label{prop:complexity_l2}
\end{proposition}
\textbf{In the remaining of this work, for considerations of computational efficiency, we choose to work in the metric space $\Psi$ and $\Omega$ defined in Remark \ref{remark:complexity} and Proposition \ref{prop:complexity_l2}, with $p=1, q=2$.}

Calculating the FNGW distance, in practice, amounts to solving a non-convex quadratic optimization problem. Following~\citet{vayer_optimal_2019}, we use the Conditional Gradient Descent (CGD) (i.e, the Frank-Wolfe algorithm). The complete algorithm is given in Alg.~\ref{algo:fngw_cal}, with the gradient with respect to the transport plan $\pi^{(i-1)}$ having the following form:
\begin{align}
	G = (1-\alpha -\beta) M(F, \tilde{F}) + 2\beta J(A, \tilde{A}) \otimes\pi^{(i-1)} + 2\alpha L(E, \tilde{E}) \otimes\pi^{(i-1)} 
\end{align}

\begin{algorithm}[!tb]
    \caption{Computation of the FNGW Distance by CGD}	
	\label{algo:fngw_cal}
\begin{algorithmic}
    \STATE {\bfseries Input:}  $g = (F, A, E, \boldsymbol{p}) $, $\tilde{g} = (\tilde{F}, \tilde{A}, \tilde{E}, \tilde{\boldsymbol{p}}) $ and trade-off parameters $(\alpha, \beta)$
    \STATE {\bfseries Init:} $\pi^{(0)} = \boldsymbol{p}\tilde{\boldsymbol{p}}^{\mathsf{T}} \in \mathbb{R}^{n\times m}$
    \FOR{$i=1,\dots, N$}
    \STATE Calculate gradient: $G = \nabla_{\pi^{(i-1)}} \mathcal{E}_{\alpha, \beta}(\{F, A, E\}, \{\tilde{F}, \tilde{A}, \tilde{E}\}, \pi^{(i-1)})$
    \STATE Solve the optimization problem with an OT solver: $\tilde{\pi}^{(i-1)} \in \argmin_{\tilde{\pi}\in \Pi(\boldsymbol{p}, \tilde{\boldsymbol{p}})}\langle G, \tilde{\pi}  \rangle$
    \STATE Update the optimal plan: $\pi^{(i)} = (1 - \gamma^{(i)})\pi^{(i-1)} + \gamma^{(i)} \tilde{\pi}^{(i-1)}$
		with $\gamma^{(i)} \in (0, 1)$ given by line-search algorithm (See details in Appendix B). 
    \ENDFOR
    \STATE Calculate the distance: $\mathrm{FNGW}_{\alpha, \beta}(g, \tilde{g}) = \mathcal{E}_{\alpha, \beta}(\{F, A, E\}, \{\tilde{F}, \tilde{A}, \tilde{E}\}, \pi^{(N)})$
    \STATE {\bfseries Output:} $\mathrm{FNGW}_{\alpha, \beta}(g, \tilde{g})$ and $\pi^{(N)}$
\end{algorithmic}
\end{algorithm} 

\section{Learning with Fused Network Gromov-Wasserstein Barycenter}
OT is the source of a continuously increasing array of tools used to apply machine learning methods to structured data. We set to demonstrate the interest of our proposed distance by showing its applicability to very different learning problems on graphs.
Within this section, we are given a set of graphs $\{g_i\}_{i=1}^{n}$ as the input of learning algorithms where $g_i = (F_i, A_i, E_i, \boldsymbol{p_i}) \in  \mathbb{R}^{n_i\times S} \times  \mathbb{R}^{n_i\times n_i} \times  \mathbb{R}^{n_i \times n_i \times T} \times \Sigma_{n_i} $ is a graph of $n_i$ nodes.

\subsection{Fused Network Gromov-Wasserstein Barycenter}
For many applications, such as graph clustering~\citep{peyre_gromov-wasserstein_2016, vayer_optimal_2019} or graph prediction~\citep{brogat-motte_learning_2022}, it is extremely useful to be able to compute a GW-based barycenter. In this section, we define a barycenter based on our proposed FNGW distance and describe an algorithm to compute it.

\begin{definition}[FNGW Barycenter]
Given a set $\{g_k\}_{k=1}^{K}$ 
and a set of weights $\{\lambda_k\}_{k=1}^K$ such that $\sum_k\lambda_k = 1$, the FNGW Barycenter for a pre-defined 
histogram $\boldsymbol{p} \in \Sigma_n$ is defined as follows:
\begin{gather}
	 \mathrm{Bary}(\{\lambda_k\}_k, \{g_k\}_{k}, \boldsymbol{p}) = \argmin_{F \in \mathbb{R}^{n\times S}, A \in \mathbb{R}^{n\times n}, E\in \mathbb{R}^{n\times n\times T}} \sum_k \lambda_k \mathrm{FNGW}_{\alpha, \beta}((F, A, E, \boldsymbol{p}), g_k)
 \label{eq:fngw:min_bary}
\end{gather}
\end{definition}
The above optimization problem can be reformulated as:
\begin{equation}
	\min_{A , E, F ,  (\pi_k \in \Pi(p, p_k))_k} \sum_k \lambda_k \mathcal{E}_{\alpha, \beta}\left((F, A, E), (F_k, A_k, E_k), \pi_k\right) 
	\label{eq:fngw:min_bary_2}
\end{equation}
To obtain the FNGW barycenter, we employ the Block Coordinate Descent (BCD) algorithm, which means that we carry out the minimization in Equation \ref{eq:fngw:min_bary} iteratively with respect to $\{\pi_k\}_k$, $F$, $A$ and $E$. The minimization with respect to $E$ has a closed form, which is given in the following proposition:
\begin{proposition}
	\label{prop:bary_c}
 Optimizing Equation \ref{eq:fngw:min_bary_2} with respect to tensor $E$
 has a closed-form solution:
    \begin{equation}
        E = \frac{1}{\mathcal{I}_{n\times T} \times_2 \boldsymbol{p} \boldsymbol{p}^{\mathsf{T}}}\sum_k\lambda_{k} (E_k \times_2 \pi_k) \times_1 \pi_k
    \label{eq:c_update}
    \end{equation}
\end{proposition}
\begin{remark}
    A regularization term $\gamma \|A\|_{1, 1}$ can be added to encourage sparsity on the barycenter matrix $A$. Due to this term, and contrarily to~\citet{vayer_optimal_2019}, a closed-form solution can no longer be obtained for $A$. Instead, we can use the proximal gradient algorithm to update $A$ at each iteration where the proximity operator $\mathrm{prox}_{\gamma\|.\|}$ coincides with the soft thresholding function $S_{\gamma}$ defined as $S_{\gamma} (A) = \mathrm{sign}(A)\times \mathrm{max}(0, A - \gamma)$ . The complete optimization algorithm is detailed in Alg. \ref{algo:barycenter}.
\end{remark}

We can notice that the optimization procedure preserves an interesting property for the tensor $E$ of the resulting barycenter:

\begin{algorithm}[tb]
\caption{Computation of FNGW Barycenter with BCD}
\label{algo:barycenter}
\begin{algorithmic}
    \STATE {\bfseries Input:} $\{g_k\}_k $,  fixed histogram $\boldsymbol{p}$, trade-off parameter $(\alpha, \beta)$, step size $\epsilon$ of proximal gradient algorithm and sparsity regularization parameter $\gamma$.
    \STATE {\bfseries Init:} Randomly initialize $E^{(0)}, F^{(0)}$ and $A^{(0)}$.
    \FOR{$i=1,\dots, N$}
    \STATE Update $(\pi_k)_k$ with Alg. \ref{algo:fngw_cal}: $\pi_k^{(i)} = \argmin_{\pi_k \in \Pi(\boldsymbol{p}, \boldsymbol{p_k})} \mathcal{E}_{\alpha, \beta}\left((F^{(i-1)}, A^{(i-1)}, E^{(i-1)}), g_k, \pi_k\right)$
    \STATE Update $E$ with Proposition \ref{prop:bary_c} and $F$ (Following Equation 8 in \citet{cuturi_fast_2014}).
    \STATE Update $A$ with proximal gradient descent:
    \FOR{$t=1,\dots, T$}
     \STATE $A^{t+1} = S_{\epsilon\gamma}(A^t - \epsilon \nabla_{A}(\sum_k \mathcal{E}_{\alpha, \beta}))$ with $\nabla_{A}(\sum_k \mathcal{E}_{\alpha, \beta}) = 2 \beta \sum_k \lambda_k(A \circ \boldsymbol{p} \boldsymbol{p}^T - \pi_k^{(i)}F_k\pi_k^{(i)T})$
    \ENDFOR
    \ENDFOR
    \STATE {\bfseries Output:} The barycenter $(F^{(N)}, A^{(N)}, E^{(N)})$
\end{algorithmic}
\end{algorithm}

\begin{proposition}
	If the set of tensors $(E_k)_k$ satisfies the condition: $\forall i, j,k, \sum_t^T E_k(i,j, t) = a \in \mathbb{R}$, 
	then the barycenter $E$ given by Algorithm \ref{algo:barycenter} also verify the same property.
    \label{prop:bary_simplex}
\end{proposition}

\begin{remark}
	When the set of edge features for the graphs $\{g_k\}_{k=1}^{K}$ lies in a simplex space, the above proposition gives us the guarantee that the edge features of their barycenter will also be a simplex. 
\end{remark}
The proofs of Proposition \ref{prop:bary_c} and \ref{prop:bary_simplex} can be found in Appendix A.
\subsection{Dictionary Learning with the {FNGW Barycenter}}

We now describe a learning task that exploits the notion of FNGW-barycenter. Dictionary learning based for graphs is an extension of factorization methods to graphs and was popularized by the seminal works of \citep{xu_gromov-wasserstein_2020} and \citet{vincent-cuaz_online_2021}. In our setting, given a set of labeled graphs $\{g_i\}_{i=1}^n$
, we want to learn a dictionary of \textit{atoms} $\{\overline{g}_s\}_{s=1}^S$
so that each graph of the training set can be reconstructed as a FNGW-barycenter of the dictionary atoms.

Assuming that the probability distributions for the atoms $\{\overline{\boldsymbol{p_s}}\}_{s=1}^S$ are fixed beforehand, our dictionary learning problem writes as:
\begin{align}
    &\min_{\substack{\{\boldsymbol{w_i}\}_{i=1}^n s.t.  \boldsymbol{w_{i}} \in \Sigma_S \\ \{(\overline{E}_s, \overline{A}_s, \overline{F}_s )\}_{s=1}^S}} 
    \sum_{i=1}^{n} \mathrm{FNGW}_{\alpha, \beta} \Big(g_i, \mathrm{Bary}(\boldsymbol{w_i}, \{\overline{g}_s\}_{s=1}^S , \boldsymbol{p_i})\Big) - \lambda \sum_i^n \|\boldsymbol{w_i}\|_2^2
    \label{eq:dl_sim}
\end{align}

where the weights $\{\boldsymbol{w_i}\}_i$ describe the \textit{embeddings} of our graphs in the dictionary and $\lambda$ is a regularization parameter which controls their sparsity.

\begin{remark}
It should be noted that this problem setting corresponds to the dictionary learning in \citet{xu_gromov-wasserstein_2020} rather than the one presented in~\citet{vincent-cuaz_online_2021}, where a graph is {projected on a linear representation of atoms}. Here, representing a graph as a FNGW-barycenter of the atoms allows in particular to use atoms comprising various number of nodes.
\end{remark}

We propose to solve the optimization problem of Equation~\ref{eq:dl_sim} with a stochastic iterative procedure, which for each sampled batches of data 
alternatively updates the embeddings $\{\boldsymbol{w_i}\}_i$ and atoms $\{(\overline{E}_s, \overline{A}_s, \overline{F}_s )\}_{s=1}^S$.
Finding the embedding $\{\boldsymbol{w_i}\}_i$ for a fixed dictionary requires an intermediate procedure called \textit{unmixing}, which is in our case a bi-level optimization problem, since the reconstructed graph is the solution of a FNGW-barycenter problem. Assuming a small change in $\boldsymbol{w_i}$ will not affect the solution of the barycenter problem, we propose therefore to solve the latter first, and find the optimal $\boldsymbol{w_i}$ with the fixed OT plans of the barycenter. The detailed algorithm along with the unmixing procedure are described in Appendix B.

\subsection{Supervised Graph Prediction with the {FNGW Barycenter}}

A Supervised Graph Prediction task consists in learning to predict an output graph $g$ from a given input $x$ in input space $\mathcal{X}$.  An appealing and original application of Optimal Transport for graphs consists in using a Gromov-Wasserstein distance as a loss in this supervised task~\citep{brogat-motte_learning_2022}. 

Let us define $\mathcal{G}$, the set of labeled graphs with maximum number of nodes $m_{\mathrm{max}}$. 
\begin{align}
    \nonumber \mathcal{G} = \big\{g = (F, A, E, \boldsymbol{p}) \hspace{0.25em} | \hspace{0.25em} & m_g \leq m_{\mathrm{max}}, \hspace{0.25em} C \in \{0, 1\}^{m_g \times m_g}, \hspace{0.25em} F = (F_i)_{i=1}^{m_g} \in \mathcal{F}^{m_g}, \\
    & E = (E_{i j}) \in \mathcal{T}^{m_g \times m_g}, \hspace{0.25em} \boldsymbol{p} ={m_g}^{-1} \mathds{1}_{m_g} \big\}
    \label{eq:g_stru}
\end{align}
where  $\mathcal{F} \subset \mathbb{R}^S$ and $\mathcal{T} \subset \mathbb{R}^T$ are finite node and edge features spaces, and $m_{\mathrm{max}}$ is the upper graph size. Denote $\mathcal{G}_m$ the relaxed version of set $\mathcal{G}$:\begin{align}
    \nonumber \mathcal{G}_m = \big\{(F, A, E, \boldsymbol{p}) \hspace{0.25em} | \hspace{0.25em} & C \in [0, 1]^{m\times m}, \hspace{0.25em} F = (F_i)_{i=1}^m \in \mathrm{Conv}(\mathcal{F})^m, \\
    & E = (E_{i j}) \in \mathrm{Conv}(\mathcal{T})^{m\times m}, \hspace{0.25em}\boldsymbol{p} =m^{-1} \mathds{1}_m \big\} 
\end{align}
where $\mathrm{Conv(\cdot)}$ denotes the convex hull of the set.

We consider a set of training pairs consisting of inputs and graphs to be predicted  $\{(x_i, g_i)\}_{i=1}^n$ drawn from a fixed but unknown distribution $\rho$ on $\mathcal{X} \times \mathcal{G}$.
We are interested in the relaxed supervised graph prediction problem, i.e., finding a estimator $f: \mathcal{X} \to \mathcal{G}_m$ of the minimizer $f^*$ of the expected risk $\mathcal{R}(f) = \mathbb{E}_{\rho}[\fngw(f(X), G)]$ where the $\mathrm{FNGW}_{\alpha, \beta}$-distance's definition is extended to $\mathcal{G}_m \times \mathcal{G}$.
We propose, as a solution to this problem, an estimator based on surrogate least square regression \citep{ciliberto_general_2020, brogat-motte_learning_2022} that expresses as a barycenter of the output training data weighted by a function $\gamma$ of the input $x$:
\begin{equation}
    \hat{f}(x) = \argmin_{g\in \mathcal{G}_m} \sum_{i=1}^n \gamma(x)_i \fngw (g, g_i)
    \label{eq:krr_fngw}
\end{equation}
with $\gamma(x) = (\mathbf{K} + n\lambda I)^{-1}\mathbf{K}_x$, where $\mathbf{K}$ is the Gram matrix of the positive definite kernel $k$ defined on the input space $\mathcal{X}$ and $\lambda > 0$ is a ridge regularization parameter. Note that for a given $x$, $\hat{g} \in \mathcal{G}$ is obtained by discretization of $\hat{f}(x)$.

The proposed estimator relies on the Implicit Loss Embedding (ILE) condition\citep{ciliberto_general_2020} that expresses that a (ILE) loss writes as a inner product of feature maps in some well chosen Hilbert space. The ILE framework justifies the relevance of solving a surrogate regression problem in the so-called output feature space and ensures theoretical guarantees for this class of estimators. The next proposition guarantees that the estimator proposed in Equation \ref{eq:krr_fngw} fits ILE condition.
\begin{proposition}
    The FNGW loss admits an Implicit Loss Embedding (ILE).
    \label{prop:ile}
\end{proposition}
Informally, this implies that our estimator is universally consistent and its learning rate is of order $n^{-1/4}$ with additional assumptions. The proof is given in Appendix A.
\section{Empirical evaluation of learning with FNGW-distance and barycenter}

In this section, we assess the relevance of the proposed distance on a diverse set of experiments, including graph classification tasks, barycenter computation and dictionary learning on synthetic datasets and  a real-world graph prediction. A Python implementation of the algorithms presented for these tasks can be found on github\footnote{Anonymized}. We leave additional clustering experiments for Appendix D.

\subsection{Labeled Graph Classification}
Supervised Classification is a classic task in learning at the graph-level. We study the potential benefit of our FNGW distance when input graphs to be classified are represented by kernels based on the FNGW-distance, using a variety of datasets where both node and edge features are present: Cuneiform~\citep{kriege_recognizing_2018}, MUTAG~\citep{debnath_structure-activity_1991, kriege_subgraph_2012}, PTC-MR~\citep{helma_predictive_2001, kriege_subgraph_2012}, BZR-MD, COX2-MD, DHFR-MD and ER-MD~\citep{sutherland_spline-fitting_2003, kriege_subgraph_2012}. All these datasets were collected by~\citet{kersting_benchmark_2016}. 

\paragraph{Experimental Settings.} For datasets whose node features are discrete, as in~\citet{vayer_fused_2020}, we use Weisfeiler-Lehman (WL) labeling to transform them into vectors of dimension $K$, being the number of iterations in WL labeling. We use a dataset-specific random vector as feature for empty edges.
To compute the FNGW distance, we set $A$ to be the matrix of shortest path distance between the nodes. Regarding the distance $d_{\Psi}$, we use the Hamming distance when WL labeling is used, or the $\ell^2$-norm otherwise. The distance $d_{\Omega}$ between edge features is 
the $\ell^2$-norm in all cases.
Our classifier is a Support Vector Machine (SVM) for which the kernel gram matrix $K$ is computed via a Gaussian kernel based on the FNGW distance, with $K_{i,j} = \exp(-\gamma \mathrm{FNGW}(g_i, g_j))$.
As baselines, we use SVMs with diverse kernels, including the Neighborhood Subgraph Pairwise Distance Kernel (NSPDK), which takes into account both node features and (discrete) edge features, while the others can only use the nodes.
We also compare our distance with the FGW distance, using the same matrices $F$ and $A$. 
With all methods, we conduct 50 times 10-fold CV on all datasets, except Cuneiform on which 5-fold CV is performed due to its size.
Each time, the dataset is split into a training set and a test set with 10\% of the data held out for the latter one. The splitting is kept the same for all the methods for a fair comparison. 
We report the mean and the standard deviation of the test accuracy. Other kernel-based baselines, as well as CV hyperparameters, are detailed in Appendix C.

\paragraph{Experimental Results}

\addtolength{\tabcolsep}{-4pt}  

\begin{table*}[thb!]
\caption{Graph classification performance on real datasets. The best results are highlighted in \textbf{Bold}.}
\centering
\resizebox{\textwidth}{!}{
\begin{tabular}{lccccccccc}
  \toprule
  \textbf{Methods} & \textbf{Cuneiform} & \textbf{MUTAG} & \textbf{PTC-MR} & \textbf{BZR-MD} & \textbf{COX2-MD} & \textbf{DHFR-MD} & \textbf{ER-MD}\\
  \midrule
  SPK~\citep{borgwardt_shortest-path_2005} & - & $84.00 \pm 8.80$ & $57.37 \pm 8.16$ & $71.10\pm 6.71$ & $64.90 \pm 8.43$ & $65.85 \pm 5.87$ & $69.78 \pm 6.14$\\
  RWK~\citep{gartner_graph_2003} & - & $88.42 \pm 7.59$ & $57.43 \pm 7.61$ & $71.81 \pm 8.48$ & $62.65 \pm 7.93$ & $67.20 \pm 5.67$ & $70.22 \pm 6.32$\\
  WLK~\citep{vishwanathan_graph_2010}  & - & $88.00 \pm 7.95$ & $61.71 \pm 6.18$ & $63.10 \pm 6.92$ & $56.71 \pm 7.82$ & $66.65 \pm 5.54$ & $66.98 \pm 6.19$\\
  GK~\citep{shervashidze_efficient_2009} & - &$83.89 \pm 8.04$& $54.86 \pm 8.24$ & $48.13 \pm 6.73$ & $44.90 \pm 6.94$ &$67.45 \pm 6.33$ & $59.16 \pm 6.44$\\
  HOPPERK~\citep{feragen_scalable_2013}   & $36.07 \pm 9.65$ & - & - &- &- & - & -\\
  PROPAK~\citep{neumann_propagation_2016}   & $11.93\pm 6.03$ & $76.84 \pm 9.59$ & $58.57 \pm 8.09$ & $73.10 \pm 7.03$ & $58.58 \pm 8.45$ &  $66.85 \pm 6.14$ & $66.44 \pm 6.34$\\
  \midrule
  NSPDK~\citep{costa_fast_2010}  & - & $ \boldsymbol{88.63 \pm 5.99}$ & $59.94 \pm 8.33$ & $73.10 \pm 7.03$ & $60.00 \pm 8.70$ & $67.30 \pm 6.46$ & $63.16 \pm 7.41$\\
  \midrule
  FGW & $80.37 \pm 7.60 $ &$86.11 \pm 8.34$ &  $ 63.77 \pm 9.20$ & $73.61 \pm 6.60$ & $64.90 \pm 7.43$ & $62.70\pm 5.95$ & $72.09 \pm 7.14 $ \\
  FNGW & $\boldsymbol{84.44 \pm 7.07}$ & $83.05 \pm 8.95$ & $\boldsymbol{63.94 \pm 7.54}$ & $\boldsymbol{75.68 \pm 7.01}$ & $\boldsymbol{68.90 \pm 8.33 }$ & $\boldsymbol{73.45 \pm 7.28} $ & $\boldsymbol{79.78 \pm 5.85} $ \\
  \bottomrule
\end{tabular}
}
\label{tab:classifier}
\end{table*}

\addtolength{\tabcolsep}{4pt}  

Graph classification results are reported in Table \ref{tab:classifier}. Our FNGW-based classifier outperforms consistently graph kernel methods, including NSPDK, which also takes advantage of the edge features. Furthermore, FNGW, being at least comparable with FGW on PTC-MR, significantly improves the performance on all the others datasets except MUTAG. We conjecture that the mutagenic effect on a bacterium (MUTAG) or carcinogenicity on rodents (PTC-MR) of molecules may not be sensitive to their chemical bond type, while these features, when combined with the distance information between atoms, are important to distinguish between active and inactive molecules (BZR-MD, COX2-MD, DHFR-MD, ER-MD).

\subsection{Barycenter Computation} 
For this first toy problem, we randomly create 8 circle graphs,  
for which the number of nodes is randomly drawn between 10 and 20. Node features are scalars following a Sine function variation with additive Gaussian Noise $\mathcal{N}(0, 0.3)$. There exists two directed edges between each pair of adjacent nodes, the ascending one being colored in blue and the descending one in green. Supplementary edges are generated with probability $0.5$ 
between every pair of nodes separated by another node.  
The toy graphs are therefore node-labeled, edge-labeled, and directed. The samples are shown in Figure \ref{fig:circle_sample}. To compute the FNGW barycenter, we take $A$ as the adjacency matrix, and encode the presence of colors in the one-hot tensor $E$. We choose the number of nodes of the barycenter to be 15, while varying the sparsity regularization parameter $\gamma$. The resulting graphs are shown in Figure \ref{fig:circle_bary}. Both node and edge features are well preserved in the barycenter for $\gamma = 5\times 10^{-5}$ and $10^{-6}$, showing that an appropriate sparsity regularization 
stabilizes the structure of the edges. 

\begin{figure}[hbtp!]
    \centering
    \begin{subfigure}{.63\textwidth}
        \centering
        \includegraphics[width=\linewidth]{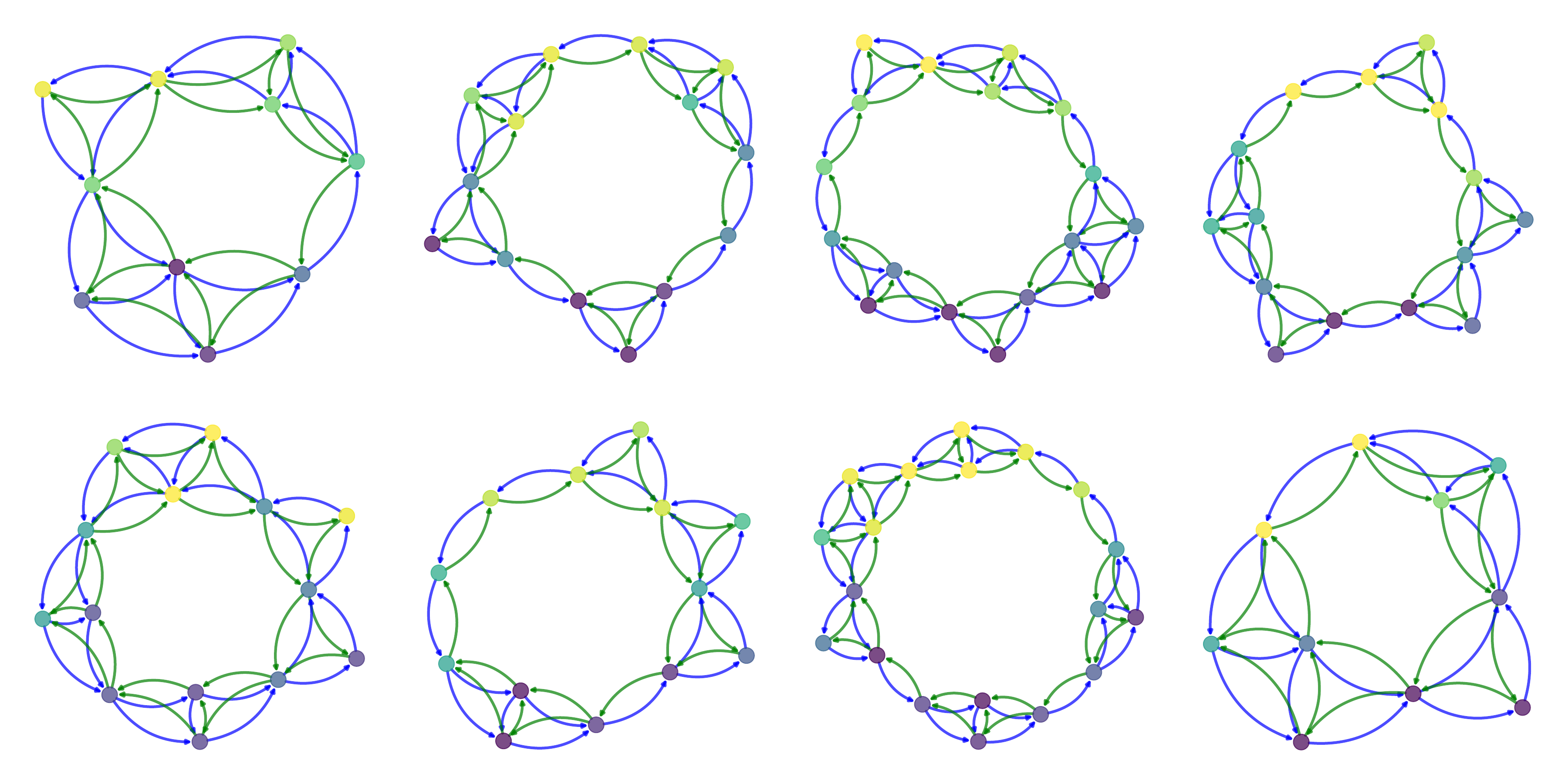}
        \caption{Random circle samples}
        \label{fig:circle_sample}
    \end{subfigure}
    \begin{subfigure}[b]{.33\textwidth}
        \centering
        \includegraphics[width=\linewidth]{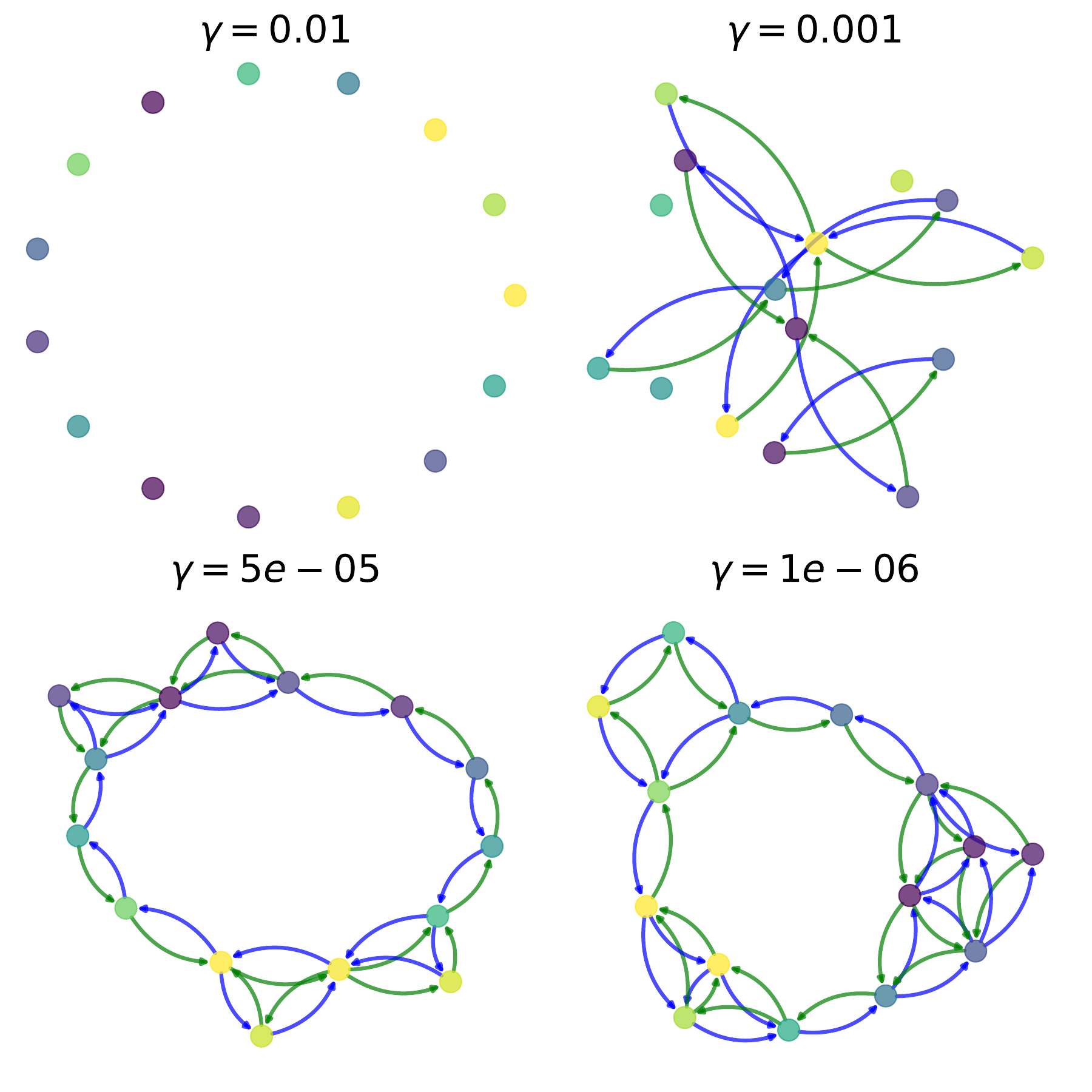}
        \caption{Barycenters with various $\gamma$}
        \label{fig:circle_bary}
    \end{subfigure}

    \caption{Illustration of graph barycenters with the FNGW distance. The color of the node indicates the scalar value of the node feature.}
    \label{fig:circle}
\end{figure}

\subsection{Representation Learning through Dictionary Learning}
We structure this second toy problem around 3 graph templates, that we generate following the Stochastic Block Model (SBM) with 1, 2 and 3 blocks. 
Each template consists of 20 nodes. For each, node features are scalars representing the block they belong to. 
Within a block, all nodes are connected with black edges. Pairs of nodes from adjacent blocks are connected with probability 0.3 by edges that are colored in blue when ascending, green when descending.
The Templates are shown in Figure \ref{fig:dl_templates}.
We repeat the following procedure to create a synthetic dataset for dictionary learning: given a vector $v \in \mathbb{R}^3$ which components are sampled from $\mathcal{N}(0, 1)$, we compute a simplex $w = \mathrm{softmax}(v)$ and generate a new graph by computing the barycenter of the templates defined above with Alg. \ref{algo:barycenter}. The number of nodes is set to be 20 for all barycenters. To control the quality of the generated barycenters, we set a maximum barycenter loss value of 0.3. The goal of the experiment is to retrieve graphs presenting the same characteristics with the 3 original templates through dictionary learning, even with a varying number of nodes. 
Setting the numbers of nodes of the atoms to 5, 10 and 15 respectively, we perform dictionary learning on the dataset. The atoms learnt are shown in Figure \ref{fig:dl_atoms}, and clearly recover the properties of the synthetic templates. 
\begin{figure}[hbtp!]
    \centering
    \begin{subfigure}{.48\textwidth}
        \centering
        \includegraphics[width=\linewidth]{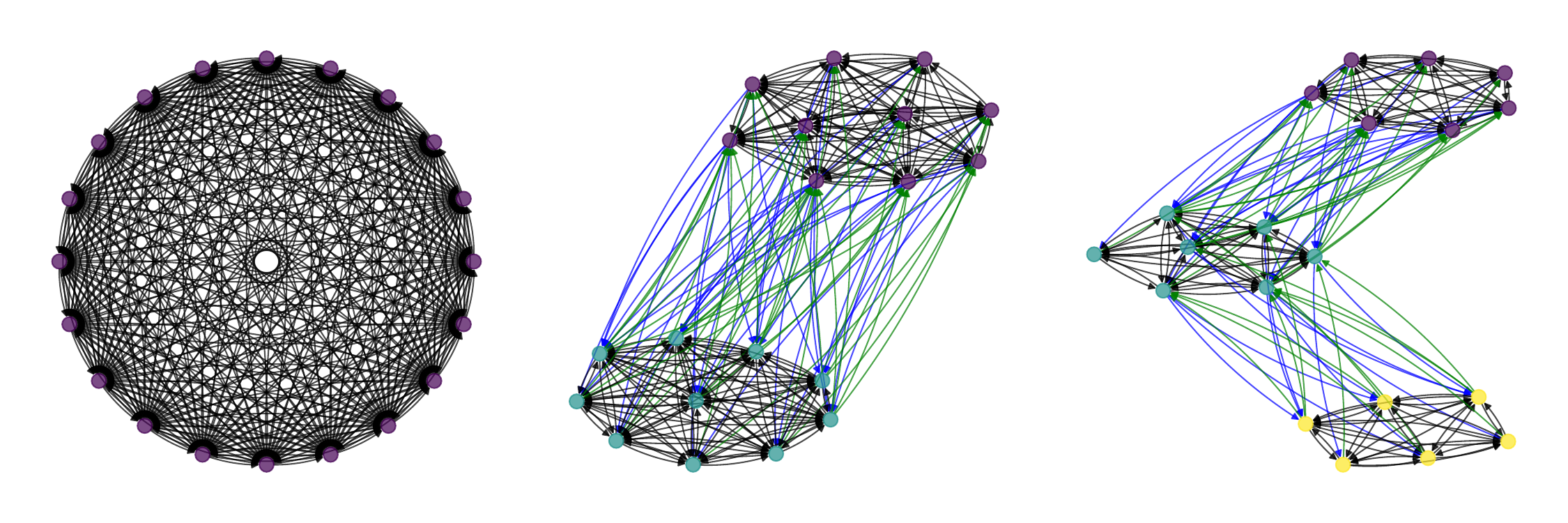}
        \caption{Synthetic templates}
        \label{fig:dl_templates}
    \end{subfigure}
    \begin{subfigure}[b]{.48\textwidth}
        \centering
        \includegraphics[width=\linewidth]{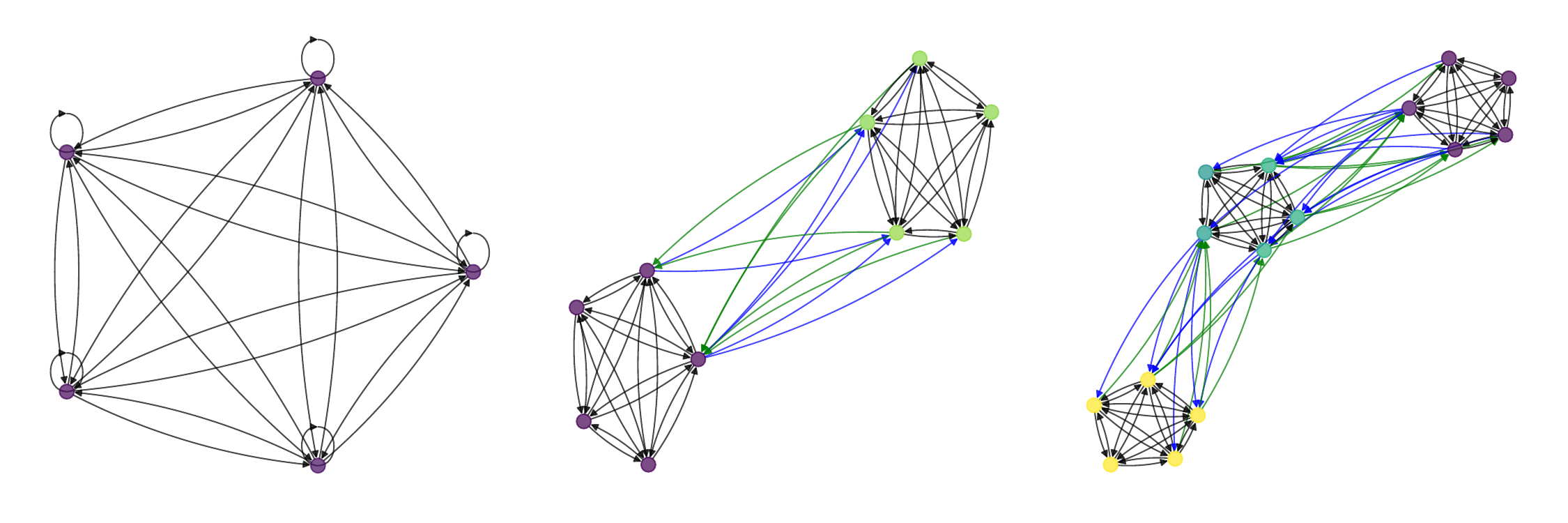}
        \caption{Learned atoms}
        \label{fig:dl_atoms}
    \end{subfigure}

    \caption{Dictionary learning on synthetic dataset.}
    \label{fig:dl_synthetic}
\end{figure}

\subsection{Supervised Graph Prediction: Metabolite Identification}

In this last experiment, we focus on a structured prediction problem, where graphs lie in the output space $\mathcal{G}$.
As in \citet{brouard2016fast, brogat-motte_learning_2022},  we use the Metabolite Identification dataset processed by \citet{duhrkop2015searching}, which consists of 4138 input tandem mass spectra extracted from the GNPS (Global Natural Products Social) public spectral library. The learning algorithm is expected to predict the metabolite (small molecules) given a tandem mass spectra. For each input spectra, a known set of metabolite candidates is provided, i.e., the inference takes place in a smaller search space $\mathcal{G}_c$. The dataset is split into a training set of size $n=3000$ and a test set of size $n_{\mathrm{test}}=1138$. 

\paragraph{Experimental Settings.}
In order to represent the metabolites, we encode their adjacency into $A$ and the atom types as one-hot vectors into $F$. Following~\citet{brogat-motte_learning_2022}, we use the matrix diffused by the normalized Laplacian of the adjacency matrix: $F_{\mathrm{diff}} = e^{-\tau\mathrm{Lap}(A)}F$.
To obtain $E$, we use three configurations: the chemical bond type, the chemical bond stereochemistry (which are embedded through one-hot encodings) and the concatenation of both (\textit{Mix}). The input representation is obtained through a probability product kernel \citep{heinonen_metabolite_2012} on the input mass spectra, which has been shown effective on this problem~\citep{brouard2016fast}. More details on hyperparameters can be found in Appendix C.

\paragraph{Experimental Results.}
\begin{table*}[t!]
\caption{Top-$k$ accuracies on the metabolite identification test set. Best results are in \textbf{Bold}.}
\begin{center}
\begin{tabular}{lccc}
\toprule
 & Top-1 & Top-10 & Top-20\\
\midrule
WL kernel & $9.8\%$ & $29.1\%$ & $37.4\%$ \\
Fingerprint with linear kernel& $28.6\%$ & $54.5\%$ & $59.9\%$ \\
Fingerprint with gaussian kernel & $\boldsymbol{41.0}\%$ & $\boldsymbol{62.0}\%$ & $\boldsymbol{67.8}\%$ \\
\midrule
FGW diffuse & $28.1\%$ & $53.6\%$ & $59.9\%$ \\
\midrule
FNGW diffuse + Bond stereo & $27.7\%$ & $55.2\%$ & $60.9\%$\\
FNGW diffuse + Bond type & $34.6\%$ & $55.1\%$ & $60.0\%$\\
FNGW diffuse + Mix & $36.2\%$ & $58.2\%$ & $61.9\%$\\
\bottomrule
\end{tabular}
\end{center}
\label{tab:expr_metabo}
\end{table*}

We measure the performance of various models on Metabolite Identification via Top-$k$ accuracy with $k \in \{1, 10, 20\}$. The results are presented in Table~\ref{tab:expr_metabo}. Results for the WL kernel, Linear fingerprint, and Gaussian fingerprint are taken from~\citep{brogat-motte_learning_2022}. We observe that with mix edge features, FNGW significantly outperforms FGW \citep{brogat-motte_learning_2022}. Our method approaches the performance of fingerprint with a Gaussian kernel, which uses an expert-derived molecular representation. This confirms that an informative representation of the output space is crucial for graph prediction.

\section{Conclusion}
We proposed a novel Optimal Transport distance for pairwise graph comparison in the presence of edge features, unlocking many applications where this information is available and relevant. This distance built on the Network Gromov-Wasserstein and fused Gromov-Wassertein distances inherits from similar geometric properties. To apply it on graph data, we devised algorithms to compute both the novel distance and the associated barycenter in the discrete case. Empirical evaluation on synthetic as well as real data shows that the use of this distance either as a tool for graph representation or as a loss function yields significant improvements in terms of performance in a large variety of learning tasks. Future works will be dedicated to scale up the distance and the barycenter algorithms to target very large graph datasets. Propagating edge information to nodes or vice versa will also be investigated.


\bibliography{references}
\bibliographystyle{plainnat}

\appendix
The supplementary material is structured as follows: 
In Appendix~\ref{app:proofs}, we present the proofs and detailed computations related to our theoretical derivations. Appendix~\ref{app:algos} complements the main part of the paper by providing additional details on our algorithms. Appendix~\ref{app:expdet} gives additional information about the data and hyperparameter used in the experiments presented in the main part of the paper. Finally, Appendix~\ref{app:expadd} includes additional experiments.

\section{Technical Proofs}
\label{app:proofs}
In Sections \ref{subsec:optcouple-proof} and \ref{subsec:metric-prop-proof}, we present the proofs of the theoretical properties of the FNGW distance that can appear as natural extensions of the properties of the NGW distance \citet{chowdhury_gromovwasserstein_2019} and the FGW distance \citet{vayer_fused_2020}.
We therefore leverage results from both \citet{chowdhury_gromovwasserstein_2019} and \citet{vayer_fused_2020}, and additionally the metric properties of space $\Omega$ for our proofs of Theorem \ref{theorem:opt_couple} and \ref{theorem:metric}. In Sections \ref{subsec:computation-proof}, \ref{subsec:bary-proof} and \ref{subsec:bary-property}, we provide proofs and calculation of the expressions we need when computing the distance itself and the barycenter. In Section \ref{subsec:ile-proof}, we demonstrate the ILE property for the FNGW distance.

For the sake of completeness, we recall first the definitions of the NGW and the FGW distance.
\begin{definition}[Network Gromov-Wasserstein Distance \citep{chowdhury_gromovwasserstein_2019}]
	Let $\mathcal{G}$ be the set of tuple of the form $(X, \varphi_X,  \mu_{X})$ where $X$ is a polish space,  $\varphi_X: X \times X \to \mathbb{R}$ is a bounded continuous measurable function and $\mu_{X}$ is a fully supported Borel probability measure.   
	Given two tuples $g_X = (X, \varphi_X,  \mu_{X})$, $g_Y=(Y,  \varphi_Y, \mu_{Y})$ from $\mathcal{G}$, the Network Gromov-Wasserstein Distance between $g_X$ and $g_Y$ is defined for any $ p \in [1, \infty]$ as follows:
	\begin{gather}
		\mathrm{NGW}_{p}(g_X, g_Y) = \min_{\mu\in \Pi(\mu_X, \mu_Y)}\mathcal{E}_{p}(g_X, g_Y, \mu)
	\end{gather}
    with 
    \begin{align}
        \mathcal{E}_{ p}(g_X, g_Y, \mu) = \bigg( \int_{X\times Y}\int_{X\times Y}   |\varphi_{X}(x, x') - \varphi_{Y}(y, y')|^p 
        d\mu(x,y)d\mu(x',y') \bigg)^{\frac{1}{p}}
    \end{align}
\end{definition}
It should be noted that the FNGW distance extends NWG into a \textit{fused} case and generalize the codomain space of $\varphi$ from $\mathbb{R}$ to a metric space. In the discrete case, when applied to graphs, we leverage this property to take into account edge labels as feature vectors.

\begin{definition}[Fused Gromov-Wasserstein Distance \citep{vayer_fused_2020}]
	Let $\mathcal{G}$ be the set of tuple of the form $(X, \psi_{X}, \varphi_X, \mu_{X})$ where $X$ is a polish space, $\psi_{X}: X\to \Psi $ is a  bounded continuous measurable function from $X$ to a metric space $(\Psi, d_\Psi)$ , $\varphi_X: X \times X \to \mathbb{R}$ is a \textbf{symmetric} bounded continuous measurable function, and $\mu_{X}$ is a fully supported Borel probability measure.   
	Given two tuples $g_X = (X, \psi_{X}, \varphi_X, \mu_{X})$, $g_Y=(Y, \psi_{Y}, \varphi_Y,  \mu_{Y})$ from $\mathcal{G}$ and trade-off parameters $\alpha \in [0, 1]$, the Fused Network Gromov-Wasserstein Distance between $g_X$ and $g_Y$ is defined for any $(p, q)\in [1, \infty]$ as follows:
	\begin{gather}
		\mathrm{FGW}_{\alpha,q, p}(g_X, g_Y) = \min_{\mu\in \Pi(\mu_X, \mu_Y)}\mathcal{E}_{\alpha, q, p}(g_X, g_Y, \mu)
	\end{gather}
    with 
    \begin{align}
        \mathcal{E}_{\alpha,  q, p}(g_X, g_Y, \mu) = \bigg(& \int_{X\times Y}\int_{X\times Y} [(1-\alpha) d_{\Psi}\left(\psi_X(x),\psi_Y(y)\right)^q  \nonumber \\ &+ \alpha |\varphi_{X}(x, x') - \varphi_{Y}(y, y')|^q] ^p 
        d\mu(x,y)d\mu(x',y') \bigg)^{\frac{1}{p}}
    \end{align}
 \label{def:fngw}
\end{definition}
Compared with FGW distance, FNGW fuses another more general function: $\omega_{X}: X\times X \to \Omega$, while releasing the symmetric constraint of $\varphi$ and allowing for feature vectors as edge labels.

\subsection{Proof of Theorem \ref{theorem:opt_couple}: Existence of FNGW Distance}\label{subsec:optcouple-proof}
Here, we recall the theorem to prove in this subsection:
\begin{customthm}{2.4}[Optimal Coupling]
	Given  $g_X = (X, \psi_{X}, \varphi_X, \omega_{X}, \mu_{X})$,  $g_Y=(Y, \psi_{Y}, \varphi_Y, \omega_{Y}, \mu_{Y})$, for any $(p, q) \in  [1, \infty]$ and  $(\alpha, \beta) \in [0, 1]^2$, there exists an optimal coupling $\mu^*\in \Pi(\mu_X, \mu_Y)$ which satisfies $\mathrm{FNGW}_{\alpha, \beta, q, p}(g_X, g_Y) = \mathcal{E}_{\alpha, \beta, q, p}(g_X, g_Y, \mu^*)$.
\end{customthm}

The proof takes mainly advantage of the Weierstrass theorem, which will use the following lemmas:
\begin{lemma}[Compactness of couplings; Lemma 10, \citet{chowdhury_gromovwasserstein_2019}] Let X, Y be two Polish spaces and let $\mu_X\in \mathrm{Prob}(X), \mu_Y \in \mathrm{Prob}(Y)$. Then $\Pi(\mu_X, \mu_Y)$ is compact in $\mathrm{Prob}(X\times Y)$.
\end{lemma}

\begin{lemma}[Continuity of the functional $\mu \mapsto \mathcal{E}_{\alpha, \beta, q, p}(g_X, g_Y, \mu)$]
For $(p, q) \in [1, \infty]^2$,  let $g_X = (X, \psi_{X}, \varphi_X, \omega_{X}, \mu_{X})$, $g_Y=(Y, \psi_{Y}, \varphi_Y, \omega_{Y}, \mu_{Y})$ both from $\mathcal{G}$, then the functional
\begin{align}
        \Pi(\mu_X, \mu_Y) & \to \mathbb{R}_{+} \bigcup +\infty \nonumber \\
        \mu & \mapsto \mathcal{E}_{\alpha, \beta, q, p}(g_X, g_Y, \mu) \nonumber
    \end{align}
is lower semicontinuous on $\Pi(\mu_X, \mu_Y)$ for the weak convergence of measures.
    
\end{lemma}
\begin{proof}
We define the functional $f: X\times Y \times X \times Y \to \mathbb{R}_{+} \bigcup +\infty$ with
\begin{align}
    f((x, y), (x', y')) = [(1-\alpha-\beta) d_{\Psi}\left(\psi_X(x),\psi_Y(y)\right)^q + \alpha d_\Omega(\omega_{X}(x, x'),\omega_{Y}(y, y'))^q \nonumber \\ + \beta |\varphi_{X}(x, x') - \varphi_{Y}(y, y')|^q] ^p 
\end{align}
then, $f$ is lower semicontinuous due to the continuity of $(d_{\Psi}, d_\Omega)$ and $(\psi_X, \psi_Y, \omega_{X}, \omega_{Y}, \varphi_{X},  \varphi_{Y})$. Using Lemma 3 from \cite{vayer_fused_2020} by considering $W = X \times Y$, which is a polish space, we can conclude $\mu \mapsto \mathcal{E}_{\alpha, \beta, q, p}(g_X, g_Y, \mu)$ is lower semicontinuous on $\Pi(\mu_X, \mu_Y)$ for the weak convergence of measures.
\end{proof}
We can now prove Theorem \ref{theorem:opt_couple}, which mainly takes advantage of the Weierstrass theorem.
\begin{proof}
    Since $\Pi(\mu_X, \mu_Y) \subset \mathrm{Prob}(X\times Y)$ is compact and  the functional $\mu  \to \mathcal{E}_{\alpha, \beta, q, p}(g_X, g_Y, \mu)$
    is lower semicontinous on $\Pi(\mu_X, \mu_Y)$, we can conclude the functional achieves its infimum for some $\mu^*$ by applying directly the Weierstrass theorem. 
\end{proof}

\subsection{Proof of Theorem \ref{theorem:metric}: Metric Properties of FNGW Distance.}\label{subsec:metric-prop-proof}

We divide Theorem \ref{theorem:metric} into four lemmas, and we suppose $g_X = (X, \psi_{X}, \varphi_X, \omega_{X}, \mu_{X})$, $g_Y=(Y, \psi_{Y}, \varphi_Y, \omega_{Y}, \mu_{Y})$ and $g_Z=(Z, \psi_{Z}, \varphi_Z,\omega_{Z}, \mu_{Z}) $ are from $\mathcal{G}$.

\begin{lemma}[Positivity]
    $\mathrm{FNGW}_{\alpha, \beta, q, p}(g_X, g_Y) \geq 0 $
\end{lemma}
\begin{proof}
    $\mathrm{FNGW}_{\alpha, \beta, q, p}(g_X, g_Y) \geq 0 $ since $d_\Psi$ and $d_\Omega$ are both metrics and $(\alpha, \beta) \in [0, 1]^2$. 
\end{proof}

\begin{lemma}[Symmetry]
    $\mathrm{FNGW}_{\alpha, \beta, q, p}(g_X, g_Y) = \mathrm{FNGW}_{\alpha, \beta, q, p}(g_Y, g_X)$
\end{lemma}
\begin{proof}
    For any $\mu\in \Pi(\mu_X, \mu_Y)$, let $\mu^{\star}:= T_{\#}\mu$ be the push forward of $\mu$ via a Borel map $T$ defined as follows
	\begin{align}
		T: X\times Y &\to Y\times X \nonumber \\
				(x, y)		&\mapsto (y, x) \nonumber
	\end{align}
	Then we have, (by the property of the push forward)
	\begin{align}
    &\mathcal{E}_{\alpha, \beta, q, p}(g_Y, g_X, \mu^{\star}) \nonumber \\
	=&\bigg(\int_{(Y\times X)^2}\big[(1-\alpha-\beta) d_{\Psi}(\psi_Y(y), \left(\psi_X(x)\right)^q  + \alpha d_\Omega(\omega_{Y}(y, y'), \omega_{X}(x, x'))^q \nonumber\\  
     &+ \beta |\varphi_{Y}(y, y') - \varphi_{X}(x, x') |^q \big]^pd\mu^{\star}(y,x)d\mu^{\star}(y',x') \bigg)^{\frac{1}{p}}  \nonumber\\
	=&\bigg(\int_{(X\times Y)^2} \big[(1-\alpha -\beta) d_{\Psi}(\psi_Y(y), \left(\psi_X(x)\right)^q  + \alpha d_\Omega(\omega_{Y}(y, y'), \omega_{X}(x, x'))^q \nonumber \\
    & + \beta |\varphi_{Y}(y, y') - \varphi_{X}(x, x') |^q \big]^pd\mu(x,y)d\mu(x',y') \bigg)^{\frac{1}{p}} \nonumber \\
	= &\bigg(\int_{(X\times Y)^2}\big[(1-\alpha -\beta) d_{\Psi}\left(\psi_X(x),\psi_Y(y)\right)^q  + \alpha d_\Omega(\omega_{X}(x, x'),\omega_{Y}(y, y'))^q \nonumber \\
    & + \beta |\varphi_{X}(x, x') - \varphi_{Y}(y, y')  |^q\big]^pd\mu(x,y)d\mu(x',y') \bigg)^{\frac{1}{p}} \nonumber \\
    = &  \mathcal{E}_{\alpha, \beta, q, p}(g_X, g_Y, \mu) \nonumber
	\end{align}
	The first equality is given by property of the push forward  ($\int fdT_{\#\mu} = \int f \circ Td\mu$ ); the second equality is given by the symmetry of $d_{\Psi}$ and $d_{\Omega}$. As a consequence, we have $\mathrm{FNGW}_{\alpha, \beta, q, p}(g_Y, g_X) = \mathrm{FNGW}_{\alpha,\beta,  q, p}(g_X, g_Y)$.
\end{proof}

\begin{lemma}[Equality]
    $\mathrm{FNGW}_{\alpha, \beta, q, p}(g_X, g_X) = 0 $. Moreover, $\mathrm{FNGW}_{\alpha, \beta, q, p}(g_X, g_Y) = 0 $ if and only there is a Borel probability space $(Z, \mu_Z)$ with measurable maps $f: Z\to X$ and  $g: Z\to Y$ such that
			\begin{gather}
		f_{\#}\mu_Z = \mu_X \label{eq:isomorphism1}\\
		g_{\#}\mu_Z = \mu_Y \\
		\| (1-\alpha - \beta) d_{\Psi}\left(\psi_X\circ f,\psi_Y\circ g\right)^q  + \alpha d_\Omega(f^{\#}\omega_{X}, g^{\#}\omega_{Y})^q + \beta |f^{\#}\varphi_X - g^{\#}\varphi_Y|^q\|_{\infty} = 0 \label{eq:isomorphism2}
	\end{gather}
	where $f^{\#}\omega_{X}: Z \times Z \to \Omega $ is defined by the map $f^{\#}\omega_{X}(z, z') = \omega_{X}(f(z), f(z'))$ and $g^{\#}$ is defined similarly.
\end{lemma}
\begin{proof}
    The proof is analogous to the proof of Theorem 18 in \cite{chowdhury_gromovwasserstein_2019}. We first deal with the case where $p\in [1, +\infty)$. For the backward direction, let us assume that there exists a Borel probability space $(Z, \mu_Z)$ and measurable maps $f: Z\to X$, $g: Z\to Y$ verifying Equation \ref{eq:isomorphism1} - \ref{eq:isomorphism2}. We consider $\mu^{\star}:= (f, g)_{\#}\mu_Z$. First of all, it is easy to prove that $\mu^{\star} \in \mathscr{C}(\mu_X, \mu_Y)$:
	\begin{gather}
		\forall A \in \mathscr{B}(X), \mu^{\star}(A\times Y) = \mu_Z((f, g)^{-1}[A\times Y]) = \mu_Z(f^{-1}[A]) = \mu_X(A) \nonumber \\
		\forall B \in \mathscr{B}(Y), \mu^{\star}(X\times B) = \mu_Z((f, g)^{-1}[X\times B]) = \mu_Z(g^{-1}[B]) = \mu_Y(B) \nonumber
	\end{gather}
	Then we have
	\begin{align}
		&\mathrm{FNGW}_{\alpha, \beta, q, p}(g_X, g_Y) \nonumber\\
		\leq &(\int_{(X\times Y)^2} \big[(1-\alpha - \beta) d_{\Psi}\left(\psi_X(x),\psi_Y(y)\right)^q  + \alpha d_\Omega(\omega_{X}(x, x'),\omega_{Y}(y, y'))^q \nonumber \\
       & +\beta |\varphi_{X}(x, x') - \varphi_{Y}(y, y')  |^q\big]^pd\mu^{\star}(x,y)d\mu^{\star}(x',y'))^{\frac{1}{p}}\nonumber \\
		= & \| (1-\alpha - \beta) d_{\Psi}\left(\psi_X\circ f,\psi_Y\circ g\right)^q  + \alpha d_\Omega(f^{\#}\omega_{X}, g^{\#}\omega_{Y})^q + \beta |f^{\#}\varphi_X - g^{\#}\varphi_Y|^q\|_{L^p(\mu_Z\otimes \mu_Z)} \nonumber 
	\end{align}
    which is equal to 0. Conversely, let $\mu^{\star} \in \mathscr{C}(\mu_X, \mu_Y)$ be the optimal coupling satisfying $\mathrm{FNGW}_{\alpha, \beta, q, p}(g_X, g_Y) = 0 $. So we have $	\| (1-\alpha - \beta) d_{\Psi}\left(\psi_X,\psi_Y\right)^q  + \alpha d_\Omega(\omega_{X}, \omega_{Y})^q + \beta |\varphi_X - \varphi_Y|^q\|_{L^p(\mu^{\star}\otimes\mu^{\star})} = 0$. To prove the existence of the desired probability space and measurable maps, we define
\begin{gather}
	Z:=X \times Y, \mu_Z:=\mu¨{\star} \nonumber \\
	f:= proj_X, g:= proj_Y \nonumber
\end{gather} 
where $proj_X: Z\to X$ and 	$proj_Y: Z\to Y$ are projection maps. It can be shown that
\begin{gather}
	\forall A \in \mathscr{B}(X), f_{\#}\mu_Z(A) = \mu^{\star}(f^{-1}[A]) =  \mu^{\star}(A\times Y) = \mu_X(A) \nonumber \\
	\forall B \in \mathscr{B}(Y), g_{\#}\mu_Z(B) = \mu^{\star}(g^{-1}[B]) =  \mu^{\star}(X\times B) = \mu_Y(B) \nonumber
\end{gather}
and 
\begin{align}
	&		\| (1-\alpha - \beta) d_{\Psi}\left(\psi_X\circ f,\psi_Y\circ g\right)^q  + \alpha d_\Omega(f^{\#}\omega_{X}, g^{\#}\omega_{Y})^q + \beta |f^{\#}\varphi_X - g^{\#}\varphi_Y|^q\|_{\infty} \nonumber \\
	= & 	\| (1-\alpha - \beta) d_{\Psi}\left(\psi_X,\psi_Y\right)^q  + \alpha d_\Omega(\omega_{X}, \omega_{Y})^q + \beta |\varphi_X - \varphi_Y|^q\|_{\infty} 
   =  0 \nonumber
\end{align}	
	The proof for the case $p = \infty$ is analogous.

 For the specific case $g_X = g_Y$, we consider $(Z, \mu_Z) = (X, \mu_X)$ and identity maps $(f, g)$, then Equation \ref{eq:isomorphism1} - \ref{eq:isomorphism2} are well verified, so we have $\mathrm{FNGW}_{\alpha, \beta, q, p}(g_X, g_X) = 0 $. The proof is thus concluded.
\end{proof}

\begin{lemma}[Relaxed Triangle Inequality]
\begin{equation}
    \mathrm{FNGW}_{\alpha, \beta, q, p}(g_X, g_Z) \leq  2^{q-1}(\mathrm{FNGW}_{\alpha, \beta, q, p}(g_X, g_Y) + \mathrm{FNGW}_{\alpha, \beta, q, p}(g_Y, g_Z)) \nonumber
\end{equation}
\end{lemma}

\begin{proof}
    Let $\mu_{ab}$ be the optimal coupling between $(g_X, g_Y)$ and $\mu_{bc}$ be the optimal coupling between $(g_Y, g_Z)$. By the Gluing Lemma, there exists a probability mesure $\tilde{\mu}\in \mathrm{Prob}(X, Y, Z)$ with marginals $\mu_{ab}$ on $(X\times Y)$ and $\mu_{bc}$ on $(Y\times Z)$. Let $\mu_{ac}$ be the marginal of $\tilde{\mu}$ on $(X\times Z)$. Due to the fact that $\mu_{ac}$ is not necessary a optimal coupling between $(g_X, g_Z)$, we have
	\begin{align}
      & \mathrm{FNGW}_{\alpha, \beta, q, p}(g_X, g_Z) \nonumber \\
		 \leq & \| (1-\alpha - \beta) d_{\Psi}\left(\psi_X,\psi_Z)\right)^q  + \alpha d_\Omega(\omega_{X},\omega_{Z})^q + \beta |\varphi_X - \varphi_Z|^q\|_{L^p(\mu_{ac}\otimes\mu_{ac})} \nonumber \\
		\leq & \| (1-\alpha - \beta) \left(d_{\Psi}\left(\psi_X,\psi_Y\right)+d_{\Psi}\left(\psi_Y,\psi_Z\right)\right)^q  + \alpha \left(d_\Omega(\omega_{X},\omega_{Y}) + d_\Omega(\omega_{Y},\omega_{Z})\right)^q \nonumber \\
     &+ \beta (|\varphi_X - \varphi_Y| + |\varphi_Y - \varphi_Z|)^q\|_{L^p(\tilde{\mu}\otimes\tilde{\mu})} \nonumber \\
		\leq  &\| (1-\alpha -\beta) 2^{q-1}\left(d_{\Psi}\left(\psi_X,\psi_Y\right)^q+d_{\Psi}\left(\psi_Y,\psi_Z\right)^q\right) + \alpha  2^{q-1}\left(d_\Omega(\omega_{X},\omega_{Y})^q+d_\Omega(\omega_{Y},\omega_{Z})^q\right) \nonumber \\
        & + \beta 2^{q-1} (|\varphi_X - \varphi_Y|^q + |\varphi_Y - \varphi_Z|^q)\|_{L^p(\tilde{\mu}\otimes\tilde{\mu})} \nonumber \\ 
		= & 2^{q-1}\|\left((1-\alpha - \beta)d_{\Psi}(\psi_X,\psi_Y)^q + \alpha d_\Omega(\omega_{X},\omega_{Y})^q + \beta |\varphi_X - \varphi_Y|^q\right) \nonumber \\
  & + \left((1-\alpha -\beta )d_{\Psi}(\psi_Y,\psi_Z)^q + \alpha d_\Omega(\omega_{Y},\omega_{Z})^q + \beta |\varphi_Y - \varphi_Z|^q \right) \|_{L^p(\tilde{\mu}\otimes\tilde{\mu})} \nonumber\\
		\leq & 2^{q-1} (\|(1-\alpha -\beta)d_{\Psi}(\psi_X,\psi_Y)^q + \alpha d_\Omega(\omega_{X},\omega_{Y})^q + \beta |\varphi_X - \varphi_Y|^q \|_{L^p(\tilde{\mu}\otimes\tilde{\mu})} \nonumber \\
  & + \|(1-\alpha - \beta)d_{\Psi}(\psi_Y, \psi_Z)^q + \alpha d_\Omega(\omega_{Y}, \omega_{Z})^q + \beta |\varphi_Y - \varphi_Z|^q \|_{L^p(\tilde{\mu}\otimes\tilde{\mu})} ) \nonumber \\
  = & 2^{q-1} (\|(1-\alpha -\beta)d_{\Psi}(\psi_X,\psi_Y)^q + \alpha d_\Omega(\omega_{X},\omega_{Y})^q + \beta |\varphi_X - \varphi_Y|^q \|_{L^p(\mu_{ab}\otimes\mu_{ab})} \nonumber \\
  & + \|(1-\alpha - \beta)d_{\Psi}(\psi_Y, \psi_Z)^q + \alpha d_\Omega(\omega_{Y}, \omega_{Z})^q + \beta |\varphi_Y - \varphi_Z|^q \|_{L^p(\mu_{ab}\otimes\mu_{ab})} ) \nonumber \\
		= & 2^{q-1}\left(\mathrm{FNGW}_{\alpha, \beta, q, p}(g_X, g_Y) + \mathrm{FNGW}_{\alpha, \beta, q, p}(g_Y, g_Z)\right) \nonumber
	\end{align}
The second inequality is the result of the triangle inequality of the inner metrics, the third inequality is due to the fact that $\forall q \geq 1 $,  $\forall a, b \geq 0 $, $(a + b)^q \leq 2^{q-1}(a^q + b^q)$. The fourth inequality is the consequence of the Minkowski’s inequality of the norm $L^p(\tilde{\mu}\otimes\tilde{\mu})$. Finally, it proves the relaxed triangle inequality with a factor of $2^{q-1}$. When $q=1$, the triangle inequality is well satisfied.  
\end{proof}

\subsection{Proof of Proposition \ref{prop:complexity_l2}: Computation Complexity Reduction for FNGW}\label{subsec:computation-proof}
For the following proofs, given tensor $E$, we define the matrix $E[t]$ by $E[t](i, j) = E(i, j, t)$ for any $i, j, t$.
\begin{customprop}{2.8}
	When $\Omega = \mathbb{R}^T$  with its associated metric $d_{\Omega}(a, b) = \|a-b\|_{\mathbb{R}^T}$ and $q=2$, the term $ L(E, \tilde{E}) \otimes\pi$ becomes
	\begin{equation}
		L(E, \tilde{E}) \otimes\pi = g(E)\boldsymbol{p}\mathbb{1}_m^\mathsf{T} + \mathbb{1}_n\boldsymbol{\tilde{p}}^\mathsf{T}h(\tilde{E})^\mathsf{T} -2 \sum_{t=1}^{T}E[t]\pi\tilde{E}[t]^\mathsf{T} 
	\end{equation}
	 where $g: \mathbb{R}^{n\times n\times T} \to \mathbb{R}^{n\times n}$ is expressed as $g(E)_{i,j}= \|E(i,j)\|^2_{\mathbb{R}^T}$, $h: \mathbb{R}^{m\times m\times T} \to \mathbb{R}^{m\times m}$ is expressed as by $h(\tilde{E})_{i,j}= \|\tilde{E}(i,j)\|^2_{\mathbb{R}^T}$.
  It can hence be computed with complexity $O(n^2mT+nm^2T)$.
\end{customprop}
\begin{proof}
	By the definition of tensor-matrix multiplication, we have
	\begin{align}
		\left(L(E, \tilde{E}) \otimes\pi\right)_{i,j} & = \sum_{k,l} \|E(i,k)- \tilde{E}(j,l)\|^2_{\mathbb{R}^T} \pi_{k,l} \nonumber \\
		& =\sum_{k,l} \sum_{t}|E(i,k, t) - \tilde{E}(j,l,t)|^2 \pi_{k,l} \nonumber \\
		& = \sum_{t} \sum_{k,l} \left(E(i,k, t)^2 + \tilde{E}(j,l,t)^2 - 2E(i,k, t)\tilde{E}(j,l,t)\right)\pi_{k,l} \nonumber \\
		& =\sum_{t} \sum_{k,l} \left(E[t](i,k)^2 + \tilde{E}[t](j,l)^2 - 2E[t](i,k)\tilde{E}[t](j,l)\right)\pi_{k,l} 
	\end{align}
 We note that the inner sum over $k$ and $l$, in the last equation above, is the same as the one that is computed in the GW distance, considering that $E[t]$ and $\tilde{E}[t]$ are the similarity matrices.
 Taking advantage of Prop.1 in \cite{peyre_gromov-wasserstein_2016}, the above Equation becomes:
	\begin{align}
		\left(L(E, \tilde{E}) \otimes\pi \right)_{i,j} &= \left(\sum_{t} E[t]^2\boldsymbol{p}\mathbb{1}_m^\mathsf{T} + \mathbb{1}_n\boldsymbol{\tilde{p}}^\mathsf{T}\tilde{E}[t]^{2\mathsf{T}} -2E[t]\pi\tilde{E}[t]^\mathsf{T} \right)_{i,j} \nonumber\\
		&= \left(g(E)\boldsymbol{p}\mathbb{1}_m^\mathsf{T} + \mathbb{1}_n\boldsymbol{\tilde{p}}^\mathsf{T}h(\tilde{E})^\mathsf{T} -2 \sum_{t=1}^{T}E[t]\pi\tilde{E}[t]^\mathsf{T}  \right)_{i,j} \nonumber
	\end{align}
 where $g: \mathbb{R}^{n\times n\times T} \to \mathbb{R}^{n\times n}$ is defined by $g(E)_{i,j}= \|E(i,j)\|^2_{\mathbb{R}^T}$ and $h: \mathbb{R}^{m\times m\times T} \to \mathbb{R}^{m\times m}$ is defined by $h(\tilde{E})_{i,j}= \|\tilde{E}(i,j)\|^2_{\mathbb{R}^T}$.
\end{proof}

\subsection{Proof of Proposition \ref{prop:bary_c}: Justification of the Barycenter Algorithm}\label{subsec:bary-proof}
\begin{customprop}{3.2}
The following optimization problem:
	\begin{equation}
		\argmin_{E\in \mathbb{R}^{n\times n\times T}} \sum_k \lambda_k \mathcal{E}_{\alpha, \beta}\left((F, A, E), (F_k, A_k, E_k), \pi_k\right)
	\end{equation}
 has a closed-form solution:
    \begin{equation}
        E = \frac{1}{\mathcal{I}_{n\times T} \times_2 \boldsymbol{p} \boldsymbol{p}^{\mathsf{T}}}\sum_k\lambda_{k} (E_k \times_2 \pi_k) \times_1 \pi_k
    \end{equation}
\end{customprop}
\begin{proof}
	Using Equation \ref{eq:tensor_matrix}, we can write
	\begin{align}
		&\argmin_{E\in \mathbb{R}^{n\times n\times T}} \sum_k \lambda_k \mathcal{E}_{\alpha, \beta}\left((F, A, E), (F_k, A_k, E_k), \pi_k\right) \nonumber \\
		= &\argmin_{E\in \mathbb{R}^{n\times n\times T}} \sum_k \lambda_k \left\langle \sum_{t} E[t]^2\boldsymbol{p}\mathbb{1}_m^\mathsf{T} + \mathbb{1}_n\boldsymbol{p_k}^\mathsf{T}E_k[t]^{2\mathsf{T}} -2E[t]\pi_k E_k[t]^\mathsf{T}, \pi_k \right \rangle \nonumber \\
		= &\argmin_{E\in \mathbb{R}^{n\times n\times T}} \sum_{t} \sum_k \lambda_k \left\langle  E[t]^2\boldsymbol{p}\mathbb{1}_m^\mathsf{T} + \mathbb{1}_n\boldsymbol{p_k}^\mathsf{T}E_k[t]^{2\mathsf{T}} -2E[t]\pi_k E_k[t]^\mathsf{T}, \pi_k \right \rangle \nonumber
	\end{align}
	Now let us write the first-order optimality condition. If $E^*$ is a minimum of the previous expression, we have:

    \begin{equation}
		\left ( \nabla_{E}  (\sum_{t} \sum_k \lambda_k \left\langle  E[t]^2\boldsymbol{p}\mathbb{1}_m^\mathsf{T} + \mathbb{1}_n\boldsymbol{p_k}^\mathsf{T}E_k[t]^{2\mathsf{T}} -2E[t]\pi_k E_k[t]^\mathsf{T}, \pi_k \right \rangle \right )_{\mid E=E^*}  = \mathbf{0} \nonumber
	\end{equation}

	which reads
	\begin{equation}
		\forall t, \sum_k \lambda_k \left (\nabla_{E[t]}   \left\langle  E[t]^2\boldsymbol{p}\mathbb{1}_m^\mathsf{T} + \mathbb{1}_n\boldsymbol{p_k}^\mathsf{T}E_k[t]^{2\mathsf{T}} -2E[t]\pi_k E_k[t]^\mathsf{T}, \pi_k \right \rangle \right )_{\mid E=E^*} = \mathbf{0} \nonumber
	\end{equation}
	We notice that for each $t$ we have the same optimization problem as the one described by Equation 12 of \cite{peyre_gromov-wasserstein_2016}. Taking advantage of Prop.~13 of \cite{peyre_gromov-wasserstein_2016} and with the notations of the tensor operations, we obtain  the desired solution.
\end{proof}

\subsection{Proof of Proposition \ref{prop:bary_simplex}: Property of the FNGW Barycenter}\label{subsec:bary-property}
\begin{customprop}{3.4}
	If the set of tensors $(E_k)_k$ satisfies the condition:	\begin{equation}
		\forall i, j,k, \sum_{t=1}^T E_k(i,j, t) = a \in \mathbb{R}
	\end{equation}
	then the barycenter $E$ given by Algorithm \ref{algo:barycenter} also verifies the same property.
\end{customprop}
\begin{proof}
By Equation\ref{eq:c_update}, we have the expression for each element of the barycenter tensor (we omit here the iteration index for the sake of clarity):  
\begin{equation}
	E(i, j, t) = \frac{1}{\boldsymbol{p}_i\boldsymbol{p}_j} \sum_k \lambda_k\sum_s\pi_k(i, s) \sum_rE_k(s, r, t) \pi_k^{\mathsf{T}}(r, j) \nonumber
\end{equation}
Summing it up along the third dimension, we have	
\begin{align}
	\sum_t E(i, j, t) &=  \sum_t \frac{1}{\boldsymbol{p}_i\boldsymbol{p}_j} \sum_k \lambda_k\sum_s\pi_k(i, s) \sum_rE_k(s, r, t) \pi_k^{\mathsf{T}}(r, j) \nonumber \\
	& = \frac{1}{\boldsymbol{p}_i\boldsymbol{p}_j} \sum_k \lambda_k\sum_s\pi_k(i, s) \sum_r \pi_k^{\mathsf{T}}(r, j)\sum_t E_k(s, r, t) \nonumber \\
	& = a \frac{1}{\boldsymbol{p}_i\boldsymbol{p}_j} \sum_k \lambda_k\sum_s\pi_k(i, s) \sum_r \pi_k^{\mathsf{T}}(r, j) \nonumber \\
	& = a \frac{1}{\boldsymbol{p}_i\boldsymbol{p}_j} \sum_k \lambda_k \boldsymbol{p}_i\boldsymbol{p}_j  = a \nonumber
\end{align}
\end{proof}

\subsection{Proof of Proposition \ref{prop:ile}: Statistical Guarantees for Supervised Graph Prediction Estimator}\label{subsec:ile-proof}
Before going through the proof, let us recall the definition of ILE property.
\begin{definition}[ILE, \citet{ciliberto_general_2020}]
    A continuous map $\ell$: $\mathcal{Z} \times \mathcal{Y} \to \mathbb{R}$ is said to admit an Implicit Loss Embedding (ILE) if there exists a separable Hilbert Space $\mathcal{H}$ and two measurable bounded maps $\psi: \mathcal{Z} \to \mathcal{H}$ and $\varphi: \mathcal{Y} \to \mathcal{H}$, such that for any $z \in \mathcal{Z}$ and $y \in \mathcal{Y}$ we have
    \begin{equation}
        \ell(z, y) = \langle \psi (z), \varphi(y) \rangle_{\mathcal{H}},
    \end{equation}
    and $\|\varphi(y)\|_{\mathcal{H}} \leq 1$. 
\end{definition}
Then we recall the $\mathrm{FNGW}_{\alpha, \beta}$-distance's definition when extended to $\mathcal{G}_m \times \mathcal{G}$. Given $g_m = \{F, E, A, \boldsymbol{p}\} \in \mathcal{G}_m$ and $g = \{\tilde{F}, \tilde{E}, \tilde{A}, \tilde{\boldsymbol{p}}\}  \in \mathcal{G}$,  
\begin{align}
    \mathrm{FNGW}_{\alpha, \beta}(g_m, g) = \min_{\pi \in \Pi(\boldsymbol{p},\tilde{\boldsymbol{p}})}  \sum_{i,j,k,l}\bigg[& \alpha \|E(i,k) - \tilde{E}(j,l)\|^2_{\mathbb{R}^T} 
        + \beta |A(i, k)- \tilde{A}(j,l)|^2 \nonumber \\
        & + (1-\alpha -\beta) \| F(i) - \tilde{F}(j)\|^2_{\mathbb{R}^S} \bigg]\pi_{k,l}  \pi_{i,j}
\end{align}

\begin{customprop}{3.7}
    The FNGW loss admits an Implicit Loss Embedding (ILE).
\end{customprop}
\begin{proof}
Using Theorem 12 from~\citep{ciliberto_general_2020}, we are going to show that $\mathrm{FNGW}_{\alpha, \beta}$ satisfies the ILE property by proving that i) $\mathcal{G}_m$ is compact, ii) $\mathcal{G}$ is finite (trivial with the definition) and iii) the function $\mathrm{FNGW}_{\alpha, \beta}(\cdot, g)$ is continuous.\\ 
    First of all, we can see that  $\mathcal{G}_m$ is compact, since $[0, 1]^{m\times m}$, $\mathrm{Conv}(\mathcal{F})^m$, and $\mathrm{Conv}(\mathcal{T})^{m\times m}$ are compact. ($\mathcal{F}$ and $\mathcal{T}$ are finite and thus compact.) Secondly, $\mathcal{G}$ is finite by definition (Equation \ref{eq:g_stru}). Now, we will prove the continuity of $\mathrm{FNGW}_{\alpha, \beta}(\cdot, g)$ for any $g\in \mathcal{G}$. Denote $dg_m=(\mathrm{d}F, \mathrm{d}A, \mathrm{d}E) \in \mathcal{G}_m$. For a given $g\in \mathcal{G}$, we have, for any $(g_m, dg_m)$: 
    \begin{align}
        & |\mathrm{FNGW}_{\alpha, \beta}(g_m + \mathrm{d}g_m, g) - \mathrm{FNGW}_{\alpha, \beta}(g_m, g) | \nonumber\\
        \leq &\sup_{\pi \in \Pi(\boldsymbol{p},\tilde{\boldsymbol{p}})} \sum_{i,j,k,l} \bigg[ \alpha \left(2\langle \mathrm{d}E(i, k), E(i, k)- \tilde{E}(j,l)\rangle_{\mathbb{R}^T} + \| \mathrm{d}E(i, k)\|^2_{\mathbb{R}^T}\right)  \nonumber \\ &+ \beta \left( 2\mathrm{d}A(i, k) \times (A(i, k)- \tilde{A}(j,l))+ \mathrm{d}A(i, k)^2 \right) \nonumber\\
        &+ (1- \alpha - \beta) \left(2\langle \mathrm{d}F(i), F(i)- \tilde{F}(j)\rangle_{\mathbb{R}^S} + \| \mathrm{d}F(i)\|^2_{\mathbb{R}^S}\right) \bigg]\pi_{k,l}  \pi_{i,j} \nonumber\\
        \leq &\sum_{i,j,k,l} \bigg[ \alpha \left(2 \|\mathrm{d}E(i, k)\|_{\mathbb{R}^T}  \|E(i, k)- \tilde{E}(j,l)\|_{\mathbb{R}^T} + \| \mathrm{d}E(i, k)\|^2_{\mathbb{R}^T}\right) \nonumber\\
        &+ \beta \left( 2|\mathrm{d}A(i, k)| \times |A(i, k)- \tilde{A}(j,l)|+ \mathrm{d}A(i, k)^2 \right) \nonumber\\
        &+(1- \alpha - \beta) \left(2 \|\mathrm{d}F(i)\|_{\mathbb{R}^S}\| F(i)- \tilde{F}(j)\|_{\mathbb{R}^S} + \| \mathrm{d}F(i)\|^2_{\mathbb{R}^S}\right) 
    \end{align}
    The first inequality is due to the fact that $\forall f, g : \Pi(\boldsymbol{p},\tilde{\boldsymbol{p}}) \to \mathbb{R}$, $|\min_{\pi} f(\pi) - \min_{\pi} g(\pi)| \leq \sup_{\pi} |f(\pi) - g(\pi)|$. The second is a direct application of the Cauchy-Schwarz inequality and the fact that $\forall i,j, \pi_{i, j} \leq 1$.
    Since $E$, $F$ and $A$ are all bounded, there exists $M\in \mathbb{R}$ such that for any $i, j,k,l$:
    \begin{align}
        M \geq  \|E(i, k)- \tilde{E}(j,l)\|_{\mathbb{R}^T} \\
        M \geq |A(i, k)- \tilde{A}(j,l)| \\
        M \geq \| F(i)- \tilde{F}(j)\|_{\mathbb{R}^S}
    \end{align}
    Then we have
    \begin{align}
        & |\mathrm{FNGW}_{\alpha, \beta}(g_m + \mathrm{d}g_m, g) - \mathrm{FNGW}_{\alpha, \beta}(g_m, g) | \nonumber\\
        \leq &\sum_{i,j,k,l} \bigg[ \alpha \left(2 \|\mathrm{d}E(i, k)\|_{\mathbb{R}^T}  M + \| \mathrm{d}E(i, k)\|^2_{\mathbb{R}^T}\right) + \beta \left( 2|\mathrm{d}A(i, k)| M+ \mathrm{d}A(i, k)^2 \right) \nonumber\\
        &+(1- \alpha - \beta) \left(2 \|\mathrm{d}F(i)\|_{\mathbb{R}^S}M + \| \mathrm{d}F(i)\|^2_{\mathbb{R}^S}\right)
        \label{eq:continue}
    \end{align}

    Now when $\|\mathrm{d}g_m\|_{\mathbb{R}^{n\times n}\times \mathbb{R}^{n\times S}\times \mathbb{R}^{n\times n\times T}} \to 0$, we have $\|\mathrm{d}F\|_{\mathbb{R}^{n\times S}} \to 0$, $\|\mathrm{d}A\|_{\mathbb{R}^{n\times n}} \to 0$, and $\|\mathrm{d}E\|_{\mathbb{R}^{n\times n \times T}} \to 0$, we can easily deduce  from Equation \ref{eq:continue} that
    \begin{equation}
        |\mathrm{FNGW}_{\alpha, \beta}(g_m + \mathrm{d}g_m, g) - \mathrm{FNGW}_{\alpha, \beta}(g_m, g) | \underset{\|\mathrm{d}g_m\| \to 0}{\longrightarrow} 0
    \end{equation}
    Hence, we have that for any $g\in \mathcal{G} $, $\mathrm{FNGW}_{\alpha, \beta}(\cdot, g)$ is continuous on $\mathbb{R}^{n\times n}\times \mathbb{R}^{n\times S} \times \mathbb{R}^{n\times n\times T}$, and thus on $\mathcal{G}_m$. 
\end{proof}
Since $\fngw$ admits an ILE, Theorems 8 and 9 from ~\citep{ciliberto_general_2020} can be instantiated on the $\fngw$-based estimator in Equation \ref{eq:krr_fngw} which gives us the following results.
\begin{theorem}[Universal Consistency]
Let $k$ be a bounded universal reproducing kernel. For any $n \in \mathbb{N}$ and any distribution $\rho$ on $\mathcal{X} \times \mathcal{G}_m$, let $f_n: \mathcal{X} \times \mathcal{G}_m$ be the estimator defined in Equation \ref{eq:krr_fngw} trained on $(x_i, y_i)_{i=1}^n$ samples independently drawn from $\rho$ and with $\lambda = n^{-1/2}$, then
\begin{equation}
    \lim_{n \to \infty} \mathcal{R} (f_n) =  \mathcal{R} (f^*) \quad \text{with probability 1}
\end{equation}
   with $\mathcal{R}(f) = \mathbb{E}_{\rho}[\fngw(f(X), G)]$ and $f^*$ denotes the bayes estimator.
\end{theorem}

\begin{theorem}[Excess-risk Bounds]
Let $\mathcal{H}$ be the Hilbert space associated with the loss $\fngw: \mathcal{G}_m \times \mathcal{G} \to \mathbb{R}_+$ in the ILE definition.
   Let $k:\mathcal{X} \times\mathcal{X}\to \mathbb{R}$ be a continuous reproducing kernel on $\mathcal{X}$ with associated RKHS $\mathcal{F}$ such that $\kappa^2 := \sup_{x\in\mathcal{X}}k(x,x)<+\infty$. Let $\rho$ be a distribution on $\mathcal{X}\times\mathcal{G}_m$ and suppose that the solution $h^*$ of the surrogate regression problem belongs to the considered hypothesis space $\mathcal{H} \otimes \mathcal{F}$.
   Let $\delta\in(0,1]$ and $n_0$ sufficiently large such that $n_0^{-1/2} \geq \frac{9\kappa^2}{n_0} \log \frac{n_0}{\delta}$. Then, for any $n \geq n_0$, the estimator $f_n$ defined in Equation \ref{eq:krr_fngw} trained on $n$ points independently sampled from $\rho$  and with $\lambda = n^{-1/2}$ is such that, with probability at least $1-\delta$
   \begin{equation}
       \mathcal{R}(f_n) - \mathcal{R}(f^*) \leq c\log(4/\delta)~ n^{-1/4},
   \end{equation}
with $c$ a constant independent of $n$ and $\delta$. 
\end{theorem}

\section{Additional details on Algorithms}
\label{app:algos}

In this section, we present additional details on our algorithms.
\subsection{Line Search in  Alg. \ref{algo:fngw_cal} }
Denoting $\Delta = \tilde{\pi}^{(i-1)} - \pi^{(i-1)}$,  $\Lambda_1 = A^2\boldsymbol{p}\mathbb{1}_m^\mathsf{T} + \mathbb{1}_n\boldsymbol{\tilde{p}}^\mathsf{T}\tilde{A}^{\mathsf{T}2}$ and $\Lambda_2 = g(E)\boldsymbol{p}\mathbb{1}_m^\mathsf{T} + \mathbb{1}_n\boldsymbol{\tilde{p}}^\mathsf{T}h(\tilde{E})^\mathsf{T}$, then we have $\pi^i=\pi^{(i-1)} + \gamma^i\Delta$ and
\begin{align}
    &\mathcal{E}_{\alpha, \beta}(\pi^i) \nonumber \\
    =& \left\langle(1-\alpha-\beta)M(F, \tilde{F}) + \beta J(A, \tilde{A})   \otimes\pi^i + \alpha L(E, \tilde{E}) \otimes\pi^i, \pi^i \right\rangle \nonumber \\
    = & \left\langle(1-\alpha-\beta)M(F, \tilde{F}) + \beta (\Lambda_1 -2A\pi^i\tilde{A}^\mathsf{T}) + \alpha (\Lambda_2 - 2 \sum_{t=1}^{T}E[t]\pi^i\tilde{E}[t]^\mathsf{T}), \pi^i \right\rangle \nonumber \\
    = & \bigg\langle(1-\alpha-\beta)M(F, \tilde{F}) + \beta (\Lambda_1 -2A(\pi^{(i-1)} + \gamma^i\Delta)\tilde{A}^\mathsf{T}) \nonumber \\
    & + \alpha (\Lambda_2 - 2 \sum_{t=1}^{T}E[t] (\pi^{(i-1)} + \gamma^i\Delta)\tilde{E}[t]^\mathsf{T}), \pi^{(i-1)} + \gamma^i\Delta \bigg\rangle 
\end{align}
Then we can rewrite $\mathcal{E}_{\alpha, \beta}(\pi^i)$ as $a  \gamma^{(i)^2} + b \gamma^{(i)} + c $ 
with \begin{gather}
    a = \left\langle -2\alpha \sum_{t=1}^{T}E[t]\Delta\tilde{E}[t]^\mathsf{T} - 2\beta A\Delta\tilde{A}^\mathsf{T}, \Delta \right\rangle \\
    b = \left\langle(1-\alpha - \beta) M(F, \tilde{F}) + \alpha (\Lambda_2- 2 \sum_{t=1}^{T}E[t]\pi^{(i-1)}\tilde{E}[t]^\mathsf{T}) + \beta(\Lambda_1 - 2A\pi^{(i-1)}\tilde{A}^\mathsf{T}), \Delta \right\rangle \nonumber \\
    + \left\langle -2 \alpha \sum_{t=1}^{T}E[t]\Delta\tilde{E}[t]^\mathsf{T} - 2 \beta A\Delta\tilde{A}^\mathsf{T}, \pi^{(i-1)}\right \rangle
\end{gather}

Hence, the line-search presented in Alg. \ref{algo:fngw_cal} can be performed with Alg. \ref{algo:fngw_line_search} using $a$ and $b$ as defined above.

\begin{algorithm}[tb]
    \caption{Line-Search in Conditional Gradient Descent}	
	\label{algo:fngw_line_search}
\begin{algorithmic}
    \STATE {\bfseries Input:}  $a$, $b$
    \IF{$a>0$}
    \STATE $\gamma^{(i)} = \min (1, \max(0, -\frac{2a}{b}))$
    \ELSE
    \IF{$a + b < 0$}
    \STATE $\gamma^{(i)} = 1 $
    \ELSE
    \STATE $\gamma^{(i)} = 0$
    \ENDIF
    \ENDIF
    \STATE {\bfseries Output:} $\gamma^{(i)}$
\end{algorithmic}
\end{algorithm}

\subsection{Details on Graph Dictionary Learning}
Let us formalize the unmixing procedure, which can be described as solving the following problem:
\begin{align}
    &\min_{\boldsymbol{w} \in \Sigma_S } 
    \mathrm{FNGW}_{\alpha, \beta} \Big(g, \mathrm{Bary}(\boldsymbol{w}, \{\overline{g}_s\}_{s=1}^S , \boldsymbol{p})\Big) - \lambda  \|\boldsymbol{w}\|_2^2
    \label{eq:unmixing}
\end{align}
The reconstructed graph is hence the solution of a FNGW-barycenter problem, which we solve first, allowing us to find the optimal $\boldsymbol{w_i}$ with the fixed OT plans of the barycenter. Our detailed algorithm is provided in Alg. \ref{algo:unmixing}. 

\begin{algorithm}[tb]
\caption{Unmixing problem solver}
\label{algo:unmixing}
\begin{algorithmic}
    \STATE {\bfseries Input:} Graph $g$, Dictionary $\{\overline{g}_s\}_{s=1}^S$
    \STATE {\bfseries Init:} Initialize uniformly $\boldsymbol{w}$.
    \REPEAT
    \STATE With $\boldsymbol{w}$ fixed, compute the barycenter with Alg.\ref{algo:barycenter}, obtaining $\tilde{g} = \mathrm{Bary}(\boldsymbol{w}, \{\overline{g}_s\}_{s=1}^S, \boldsymbol{p})$ and save $S$ independent optimal transport plans $\pi_s$ between $\tilde{g}$ and $\overline{g}_s$.
    \STATE Compute the optimal transport plan $\pi$ between $g$ and $\tilde{g}$.
    \STATE Compute the optimal $\boldsymbol{w}$ by minimizing Equation \ref{eq:unmixing} with fixed transport plans $\pi$ and $\{\pi_s\}_s$ using the CG algorithm.
    \UNTIL{Convergence of the reconstruction loss.} 
    \STATE {\bfseries Output:} $\boldsymbol{w}$, $\pi$ and $\{\pi_s\}_s$
\end{algorithmic}
\end{algorithm}

This procedure is applied to the graphs in $\{g_i\}_i$, giving us corresponding $\{\boldsymbol{w_i}\}_i$ and transport plans, which allows us to update atoms $\{(\overline{g}_s)\}_{s=1}^S$ by gradient descent.
Given a batch of graph samples $\{g_b\}_{b=1}^B$, we define the following functional, representing the batch loss: 
\begin{equation}
    \xi (\overline{F}_s, \overline{A}_s, \overline{E}_s) = \frac{1}{B} \sum_{b=1}^{B} \mathcal{E}_{\alpha, \beta}\left(\left(E_b, A_b, F_b\right), \left(\tilde{E}_b, \tilde{A}_b, \tilde{F}_b\right), \pi_b\right)
\end{equation}
with $(\tilde{E}_b, \tilde{A}_b, \tilde{F}_b)$ the reconstructed graph of the batch sample $g_b$ from the results of the unmixing problem and $\pi_b$ the optimal transport plan between them. The reconstructed graph is then expressed as:
\begin{gather}
    \tilde{E}_b = \frac{1}{\mathcal{I}_{n\times T} \times_2 \boldsymbol{p}\boldsymbol{p}^{\mathsf{T}}}\sum_s \boldsymbol{w}_{s,b} (\overline{C}_s \times_2 \overline{\pi}_{s,b}) \times_1 \overline{\pi}_{s,b} \\
    \tilde{A}_b =  \frac{1}{\boldsymbol{p}\boldsymbol{p}^{\mathsf{T}}} \circ \sum_s \boldsymbol{w}_{s,b} \overline{\pi}_{s,b}\overline{A}_s \overline{\pi}_{s, b}^{\mathsf{T}} \\
    \tilde{F}_b =  \sum_s \boldsymbol{w}_{s,b} \mathrm{diag}(\frac{1}{\boldsymbol{p}}) \overline{\pi}_{s,b}\overline{F}_s
\end{gather}
where $\overline{\pi}_{s,b}$ denotes the optimal transport plan between the reconstructed graph $\tilde{g_b}$ and atom $\overline{g}_{s}$. Then the estimated gradients is written as: 
\begin{align}
    \nabla_{\overline{F}_s}\xi & = \frac{2}{B} \sum_{b=1}^{B} \boldsymbol{w}_{s,b} \overline{\pi}_{s,b}^{\mathsf{T}} \left(\tilde{F}_b - \mathrm{Diag}(\frac{1}{\boldsymbol{p}_b})\pi_b^{\mathsf{T}} F_b \right) \label{eq:dict_atom_f} \\
    \nabla_{\overline{A}_s}\xi  & = \frac{2}{B} \sum_{b=1}^{B} \boldsymbol{w}_{s,b} \overline{\pi}_{s,b}^{\mathsf{T}}  \left(\tilde{A}_b - \frac{1}{\boldsymbol{p}_b\boldsymbol{p}_b^{\mathsf{T}}} \circ \pi_b^{\mathsf{T}}A_b\pi_b \right) \overline{\pi}_{s,b} \\ 
     \nabla_{\overline{E}_s}\xi  & = \frac{2}{B} \sum_{b=1}^{B} \boldsymbol{w}_{s,b}  \left( \left(\tilde{E}_b - \frac{1}{\mathcal{I}_{n\times T} \times_2 \boldsymbol{p}_b\boldsymbol{p}_b^{\mathsf{T}}} \circ \left(E_b\times_2 \pi_b^{\mathsf{T}}\right)\times_1 \pi_b^{\mathsf{T}} \right) \times_2 \overline{\pi}_{s,b}^{\mathsf{T}} \right) \times_1 \overline{\pi}_{s,b} \label{eq:dict_atom_e}
\end{align}
The complete algorithm for dictionary learning is summarized by Alg. \ref{algo:dl}. 

\begin{algorithm}[tb]
\caption{Stochastic Graph Dictionary Learning}
\label{algo:dl}
\begin{algorithmic}
    \STATE {\bfseries Input:} Graph dataset $\{g_i\}_{i=1}^n$
    \STATE {\bfseries Init:} Randomly initialize the atoms $\{(\overline{C}_s, \overline{A}_s, \overline{F}_s )\}_{s=0}^S$ 
    \FOR{$i=1,\dots, N$}
    \STATE Sample a mini-batch of graphs from dataset: $\{(E_b, A_b, F_b, \boldsymbol{p}_b)\}_{b=1}^B $ 
    \FOR{$b=1,\dots, B$}
        \STATE With $\{(\overline{E}_s, \overline{A}_s, \overline{F}_s )\}_{s=1}^S$ fixed, solve the unmixing problem: \\ $\boldsymbol{w}_b, \pi_b, \{\overline{\pi}_{b, s}\}_s = \text{Alg.\ref{algo:unmixing}}((E_b, A_b, F_b, \boldsymbol{p}_b), \{(\overline{C}_s, \overline{A}_s, \overline{F}_s )\}_{s=1}^S)$ 
    \ENDFOR
    \FOR{$s=1,\dots, S$}
    \STATE With fixed $\boldsymbol{w}_b, \pi_b, \{\overline{\pi}_{b, s}\}_s$, compute the gradients of the mini-batch loss with respect to $\{\overline{C}_s, \overline{A}_s, \overline{F}_s \}$.
    \STATE Update $\{\overline{C}_s, \overline{A}_s, \overline{F}_s \}$ using gradients from Equation \ref{eq:dict_atom_f} to Equation \ref{eq:dict_atom_e}. 
    \ENDFOR
    \ENDFOR
    \STATE {\bfseries Output:} The atoms of the dictionary $\{(\overline{C}_s, \overline{A}_s, \overline{F}_s )\}_{s=1}^S$
\end{algorithmic}
\end{algorithm}

\section{Additional details on Experiments}
\label{app:expdet}

This section is dedicated to additional details about the experiments carried out on real-world datasets.
\subsection{Labeled Graph Classification.}
\paragraph{Datasets.} We consider 7 graph datasets for classification: Cuneiform~\citep{kriege_recognizing_2018}, MUTAG~\citep{debnath_structure-activity_1991, kriege_subgraph_2012}, PTC-MR~\citep{helma_predictive_2001, kriege_subgraph_2012}, BZR-MD, COX2-MD, DHFR-MD and ER-MD~\citep{sutherland_spline-fitting_2003, kriege_subgraph_2012}. While Cuneiform is one of the earliest systems of writing realized by wedge-shaped marks on clay tablets, the rest of the datasets are of the nature of small molecules. 
All the datasets can be downloaded from \url{https://ls11-www.cs.tu-dortmund.de/staff/morris/graphkerneldatasets}.

\paragraph{Experimental Settings.} 
As baselines, we use kernel-based methods, including the shortest path kernel (SPK)~\citep{borgwardt_shortest-path_2005}, the random walk kernel (RWK)~\citep{gartner_graph_2003}, the Weisfeler Lehman kernel (WLK)~\citep{vishwanathan_graph_2010}, the graphlet count kernel (GK)~\citep{shervashidze_efficient_2009}, the HOPPER kernel (HOPPERK)~\citep{feragen_scalable_2013}, the propagation kernel (PROPAK)~\citep{neumann_propagation_2016} and the Neighborhood Subgraph Pairwise Distance Kernel (NSPDK)~\citep{costa_fast_2010}. 

For the classifier based on FNGW or FGW, the coefficient  $\gamma$ presented in the Gram matrix computation is cross-validated within $\{2^{i} \mid i\in \llbracket -10, 10 \rrbracket\}$ while the number of iteration $K$ in WL labeling is searched in $\llbracket 0, 4 \rrbracket $ for MUTAG and PTC-MR or $\llbracket 0, 3 \rrbracket $ for the other datasets (0 means no WL labeling is used) except Cuneiform. For FNGW, we cross-validate 7 values of the trade-off parameter $\alpha$ via a logspace search in $[0, 0.5]$ and the same strategy is applied to $\beta$, while a total of 15 values are drawn from logspaces $[0, 0.5]$ and $[0.5, 1]$ to cross-validate $\alpha$ for FGW. For the kernel-based methods, decay factor $\lambda$ of RWK is cross-validated from $\{10^{i} \mid i\in \llbracket -6, -2 \rrbracket\}$, the number of iterations of WLK is cross-validated in $\llbracket 1, 10 \rrbracket$ while the one of PROPAK is chosen from $\{1, 3, 5, 8, 10, 15, 20\}$. For GK, the CV range of the graphlet size is $\{3, 4, 5\}$ and the one of the precision level $\epsilon$ and the conﬁdence $\delta$ is $\{0.1, 0.05\}$. For NSPDK, the maximum considered radius $r$ between vertices is cross-validated within $\llbracket 0, 5 \rrbracket$, and neighborhood depth $d$ is chosen from $\{3, 4, 5, 6, 8, 10, 12, 14, 20\}$. Finally, the regularization coefficient $C$ of the SVM is cross-validated within $\{10^{i} \mid i\in \llbracket -7, 7 \rrbracket\}$ for all the methods except for the case MUTAG-FNGW where the value $10^7$ is not included.

\subsection{Metabolite Identification.}
\paragraph{Datasets.} To evaluate the performance for metabolite identification from tandem mass spectra, we use the data extracted and processed in \cite{duhrkop2015searching}, which is a set of 4,138 small compounds from the public GNPS Public Spectral Libraries (\url{https://gnps.ucsd.edu/ProteoSAFe/libraries.jsp}). The candidate sets were built with molecular structures from PubChem. The dataset can be downloaded from \url{https://zenodo.org/record/804241#.Yi9bzS_pNhE}, which is released under \textit{Creative Commons Attribution 4.0 International} license.

\paragraph{Experimental Settings.}
Due to the computational cost, during prediction, only the 5 training samples with the greatest weights $\gamma(x)_i$ are taken into account, rather than all the samples, as described in Equation $\ref{eq:krr_fngw}$. We choose the ridge regularization parameter $(\lambda = 10^{-4})$ and the diffusion rate ($\tau=0.6$) following~\cite{brogat-motte_learning_2022}. We use a separate validation set (1/5 of the training set) to select the hyperparameters $\alpha$ and $\beta$ of FNGW loss. $\alpha$ and $\beta$ are chosen from $\{0.1, 0.33, 0.5, 0.67\}$ with the constraint $\alpha + \beta < 1$. For the experiment with FGW, we keep the same hyper-parameter as in ~\citep{brogat-motte_learning_2022}. 

\section{Additional Experiments}
\label{app:expadd}

In this section, we present additional experiments illustrating the possible applications of our proposed FNGW distance.
\subsection{Computation of a Similarity Matrix}

In this first experiment, we use the global bilateral migration networks released by the World Bank and processed by~\citet{chowdhury_gromovwasserstein_2019} to test our distance. Differently than~\citep{chowdhury_gromovwasserstein_2019}, the capacity of the FNGW distance to model higher dimensional edge features allows us to take in account both networks (corresponding to the male migration and the female migration, separated) conjointly, by combining their edge features.
The resulting dataset consists of 5 graphs, each containing 225 nodes representing countries or administrative regions. The feature of each edge $(i \to j)$ is a two-dimensional vector, where each indicates the number of respectively male and female migrating from region $i$ to region $j$. Since the networks are complete and do not possess any node features, the trade-off parameter $\alpha$ of the FNGW distance was set to $1$. The weights of the nodes in the network are uniformly distributed. Figure \ref{fig:migration} shows the dissimilarity matrix of the migration networks and its associated single-linkage clustering dendrogram. 

\begin{figure*}[hbtp!]
  \centering
  \includegraphics[width=1.0\textwidth]{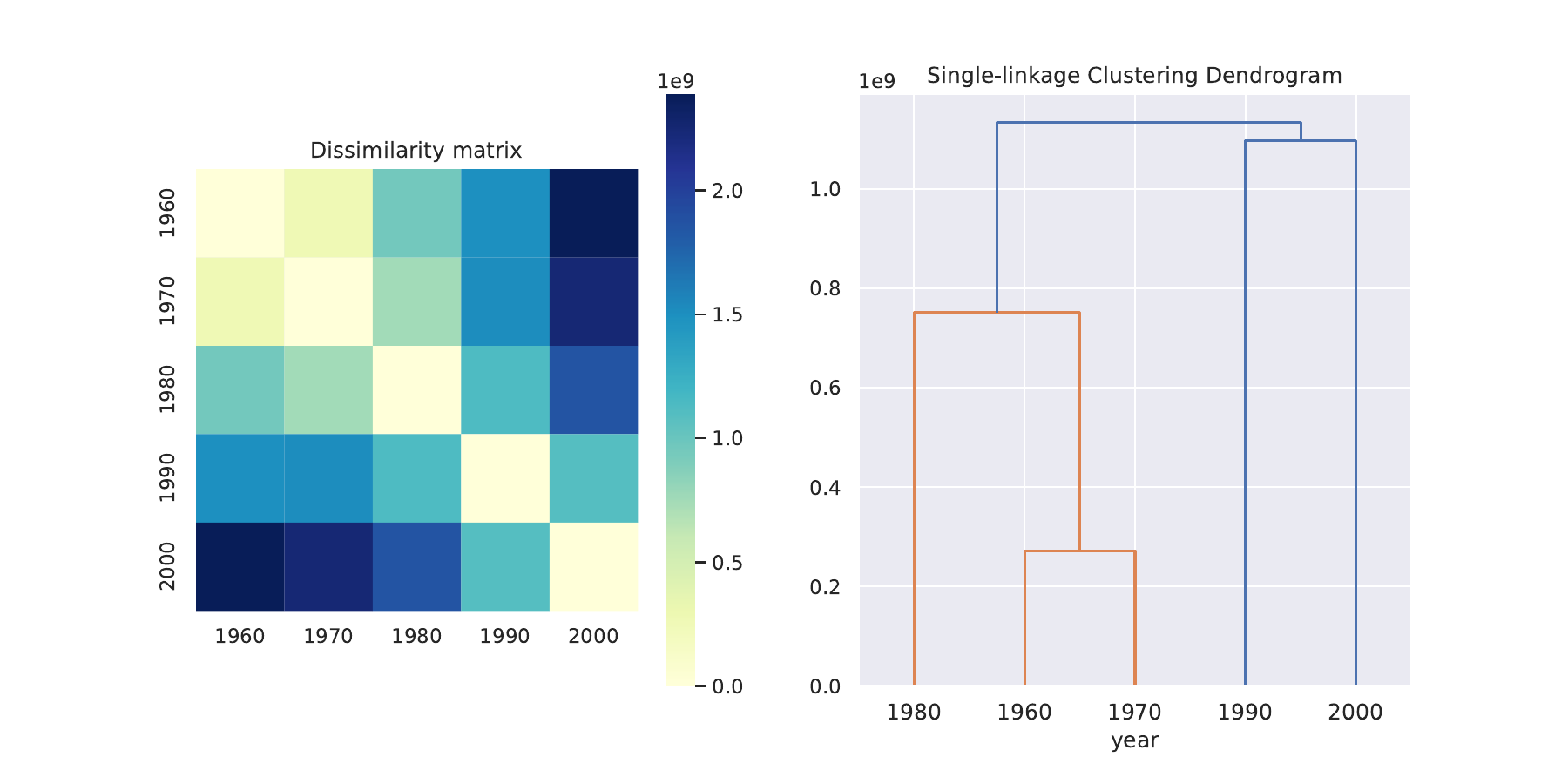}
  \caption{Results of migration networks}
  \label{fig:migration}
\end{figure*}

\subsection{Unsupervised Learning: Labeled Graphs Clustering}

In this synthetic task, we generate 3 groups of graphs, following the Stochastic Block Model (SBM) with the number of blocks in each graph chosen from $\{1, 2, 4\}$. For each cluster, we generate 15 graphs whose numbers of nodes are randomly chosen from $\{20, 30, 40\}$. For each graph, the feature of the nodes from block $i$ is sampled from $\mathcal{N}(i, 1)$. For the edges, we consider the following 3 labels: \{black, blue, green\}. Edges between nodes in the same block are labeled as black with probability 1. For every pair of nodes $(p, q)$ with $p$ from block $i$ and node $q$ from block $i+1$, there exists an edge $p\to q$ labeled as green, and there exists an edge $q\to p$ labeled as blue with the same probability.
There is no other possible edges. In total, the dataset contains 45 graphs. Several synthetic graph samples are shown in Figure~\ref{fig:cluster_synthetic}. We apply the $k$-means algorithm in order to perform graph clustering, \textbf{considering the FNGW barycenter as the centroid of each cluster and the FNGW distance as the cluster assignment metric}. Cluster centroids are randomly initialized with 20 nodes; their evolution as found by $k$-means is presented in Figure~\ref{fig:centroid}. We can observe that the resulting centroids recover not only the node features but also the characteristic of the edge labels of the different groups in the synthetic dataset.
\begin{figure}[hbtp!]
  \centering
  \includegraphics[width=1.0\textwidth]{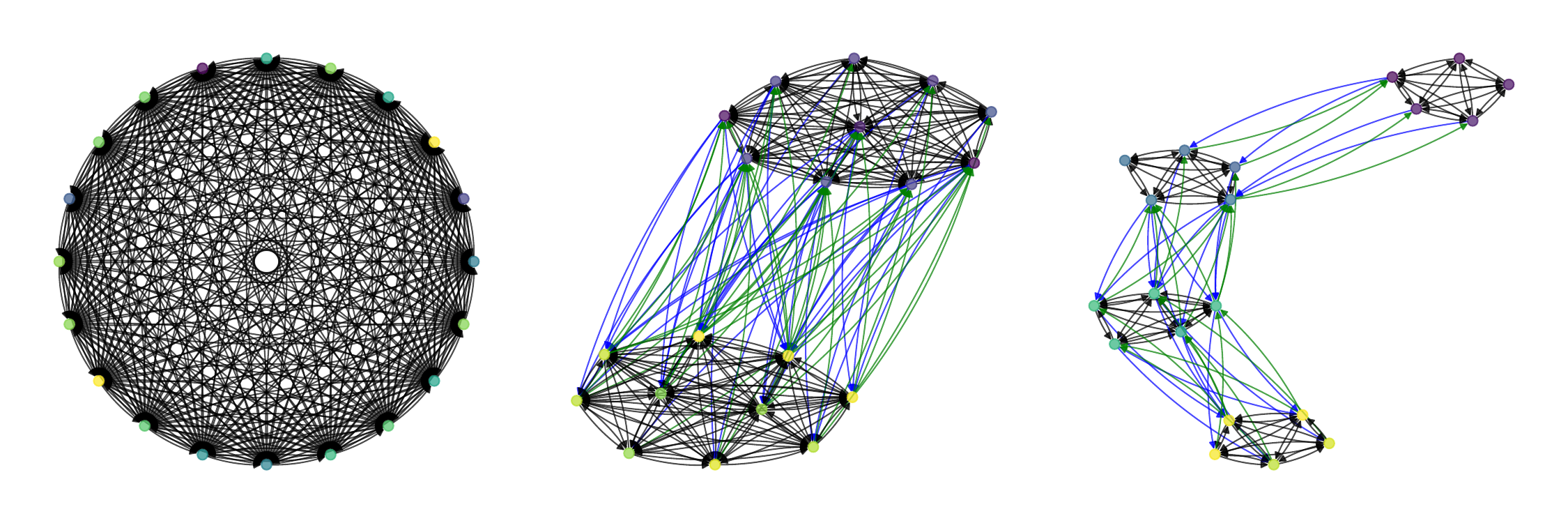}
  \caption{Samples of graphs for clustering with 20 nodes}
  \label{fig:cluster_synthetic}
\end{figure}

\begin{figure}[hbtp!]
  \centering
  \includegraphics[width=1.0\textwidth]{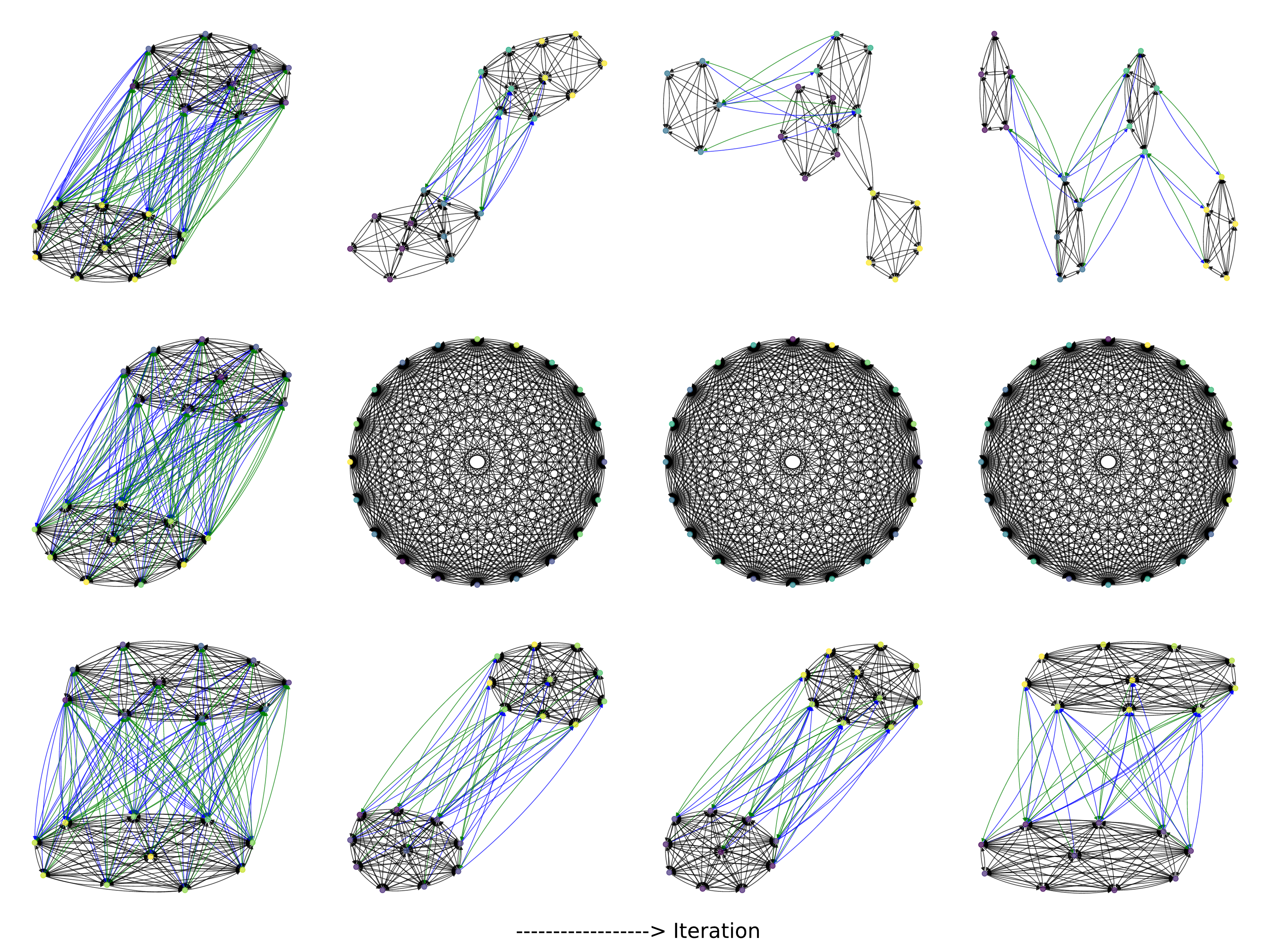} 
  \caption{Evolution of centroids of each cluster found by $k$-means algorithm.}
  \label{fig:centroid}
\end{figure}



\end{document}